\documentclass[letterpaper]{article} 
\usepackage{aaai2026}  
\usepackage{times}  
\usepackage{helvet}  
\usepackage{courier}  
\usepackage[hyphens]{url}  
\usepackage{graphicx} 
\usepackage{natbib}  
\usepackage{caption} 
\usepackage{algorithm}
\usepackage{algorithmic}
\usepackage{algorithmic}

\usepackage{adjustbox}
\usepackage{multirow}
\usepackage{booktabs}
\usepackage[table]{xcolor}
\usepackage{amsmath}
\usepackage{cleveref}

\usepackage{newfloat}
\usepackage{listings}
\DeclareCaptionStyle{ruled}{labelfont=normalfont,labelsep=colon,strut=off} 
\floatstyle{ruled}
\newfloat{listing}{tb}{lst}{}
\floatname{listing}{Listing}
\usepackage{amsfonts}
\usepackage{subcaption}
\usepackage{makecell}
\usepackage{amsthm}
\newtheorem{assumption}{Assumption}
\newtheorem{theorem}{Theorem}

\title{FedALT: Federated Fine-Tuning through Adaptive Local Training with Rest-of-World LoRA}
\author{
    Jieming Bian\textsuperscript{\rm 1}\equalcontrib \quad
    Lei Wang\textsuperscript{\rm 1}\equalcontrib \quad
    Letian Zhang\textsuperscript{\rm 2} \quad
    Jie Xu\textsuperscript{\rm 1}
}
\affiliations{
    \textsuperscript{\rm 1}University of Florida, Gainesville, FL, USA\\
    \textsuperscript{\rm 2}Middle Tennessee State University, Murfreesboro, TN, USA\\
    {jieming.bian@ufl.edu, leiwang1@ufl.edu, letian.zhang@mtsu.edu, jie.xu@ufl.edu}
}

\usepackage{bibentry}

\begin{document}

\maketitle

\begin{abstract}
Fine-tuning large language models (LLMs) in federated settings enables privacy-preserving adaptation but suffers from cross-client interference due to model aggregation. Existing federated LoRA fine-tuning methods, primarily based on FedAvg, struggle with data heterogeneity, leading to harmful cross-client interference and suboptimal personalization. In this work, we propose \textbf{FedALT}, a novel personalized federated LoRA fine-tuning algorithm that fundamentally departs from FedAvg. Instead of using an aggregated model to initialize local training, each client continues training its individual LoRA while incorporating shared knowledge through a separate Rest-of-World (RoW) LoRA component. To effectively balance local adaptation and global information, FedALT introduces an adaptive mixer that dynamically learns input-specific weightings between the individual and RoW LoRA components, drawing conceptual foundations from the Mixture-of-Experts (MoE) paradigm. Through extensive experiments on NLP benchmarks, we demonstrate that FedALT significantly outperforms state-of-the-art personalized federated LoRA fine-tuning methods, achieving superior local adaptation without sacrificing computational efficiency.
\end{abstract}


\section{Introduction}

\begin{figure*}[t]
  \centering
    \includegraphics[width=0.8\linewidth]{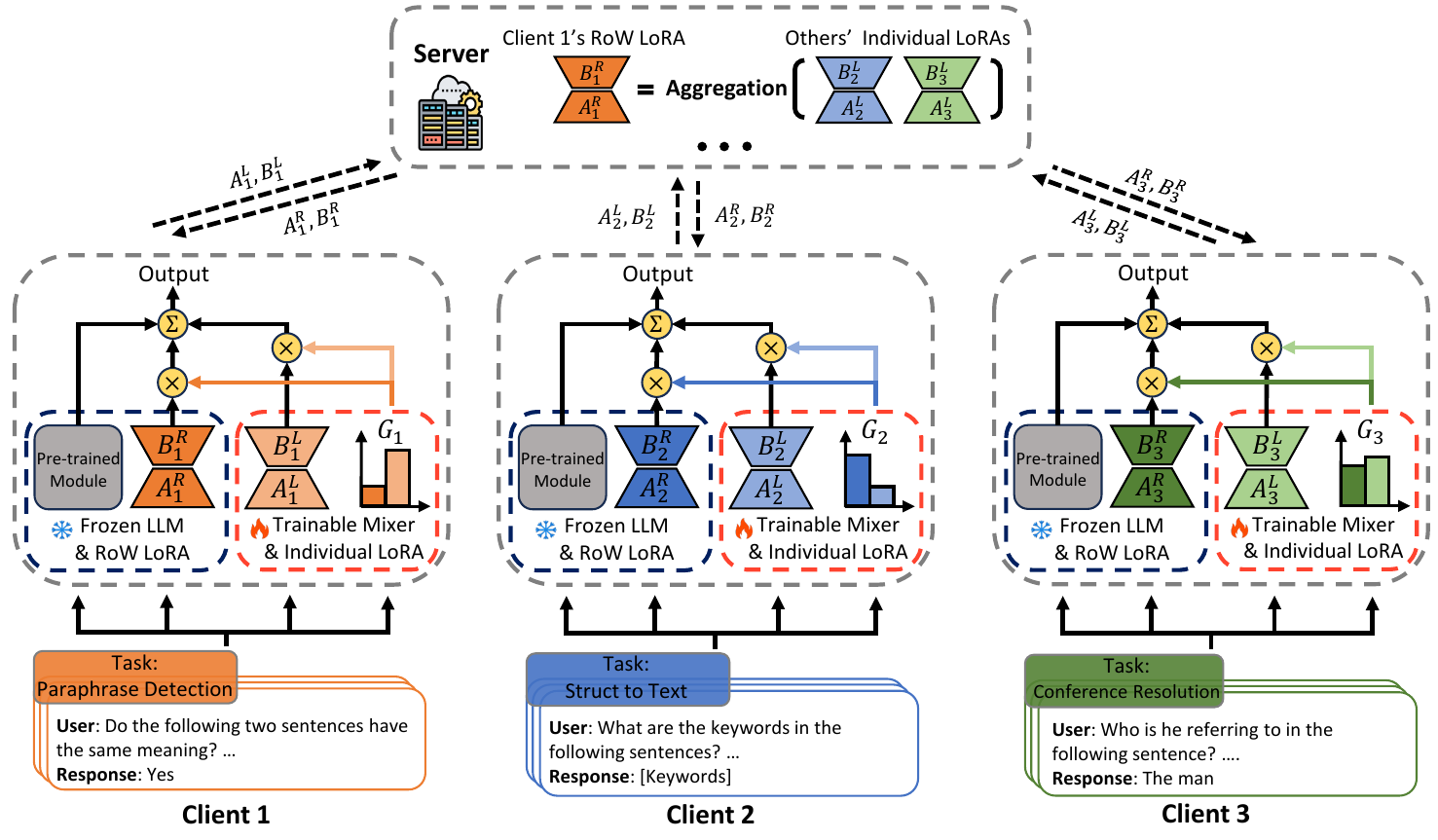}
  \caption{\textbf{Illustration of FedALT.} Instead of directly aggregating local LoRA modules from each client using FedAvg, FedALT introduces a frozen RoW LoRA component to transmit shared global knowledge while preserving client-specific adaptations through Individual LoRA. The adaptive mixer dynamically combines the RoW LoRA and Individual LoRA.}
  \label{fig:illustration}
\end{figure*}
Large language models (LLMs) \cite{kenton2019bert, brown2020language, raffel2020exploring, touvron2023llama, zhou2024comprehensive, zeng2022glm, touvron2023llama2} have demonstrated exceptional capabilities in language understanding and generation, enabling a wide range of applications. Adapting pretrained LLMs to specialized domains or further enhancing their performance requires fine-tuning with task-specific data \cite{minaee2024large}. However, full fine-tuning, which involves updating all model parameters, is computationally expensive and often impractical for real-world deployment. To address this limitation, several parameter-efficient fine-tuning (PEFT) \cite{ding2023parameter, fu2023effectiveness, liu2022few, li2021prefix, lester2021power} methods have been proposed, with low-rank adaptation (LoRA) \cite{hu2021lora} being one of the most widely used. LoRA reduces the number of trainable parameters by integrating low-rank matrices into the model, significantly reducing computational costs.

\textcolor{black}{Fine-tuning LLMs necessitates substantial data volumes for adaptation to downstream tasks; however, this requisite data often resides across multiple institutions where privacy regulations and security protocols prohibit direct data sharing. Federated Learning (FL) facilitates collaborative fine-tuning without raw data exposure, establishing itself as an essential paradigm for privacy-preserving adaptation \cite{bian2025survey}. Nonetheless, existing federated LoRA fine-tuning methodologies \cite{zhang2024towards, sun2024improving, wang2024flora, Bian_2025_ICCV} predominantly presuppose a singular global model, neglecting the inherent data heterogeneity among participating clients. Such a homogeneous modeling approach frequently yields suboptimal performance and introduces fairness disparities, particularly in real-world deployments where client data exhibits significant variance in volume, domain, and distributional characteristics. While limited research on personalized federated LoRA fine-tuning \cite{hao2025personalized, qi2024fdlora} exists, these approaches primarily address label heterogeneity cases. However, in realistic LLM fine-tuning contexts, clients' learning objectives naturally diverge due to diverse business imperatives \cite{yao2022benchmark}. Personalized federated LoRA fine-tuning addressing more realistic task-heterogeneous scenarios \cite{yang2024dual} remains insufficiently investigated in the current literature.}

\textcolor{black}{Federated Averaging (FedAvg) \cite{mcmahan2017communication} is perhaps the most widely used FL framework, serving as the foundation for numerous follow-up works that introduce various optimizations. Its core principle is to aggregate locally trained client models into a global model, which then initializes the next round of local training. Most personalized FL algorithms \cite{yang2024dual, qi2024fdlora, tan2022towards, deng2020adaptive, collins2021exploiting, fallah2020personalized}, despite their implementation differences, build upon the FedAvg framework by incorporating personalization mechanisms—and personalized federated LoRA fine-tuning follows the same paradigm. For example, the recently proposed FedDPA \cite{yang2024dual} algorithm personalizes LoRA by introducing both global and local LoRA components. The global LoRA component is trained using FedAvg, while the local LoRA component is trained independently on each client's private data. This local adaptation can occur either after the global LoRA stabilizes or alternatedly with global LoRA training, in which case the final output is computed as a weighted combination of both components using a pre-defined weighting factor. However, our motivational study reveals that FedAvg-based approaches may degrade client-specific performance due to \textcolor{black}{harmful cross-client interference}. We identify two fundamental limitations of FedAvg-based methods like FedDPA in heterogeneous FL settings. \textbf{First}, they suffer from harmful cross-client interference due to reliance on global aggregation, which can negate individual client improvements achieved during local fine-tuning. \textbf{Second}, these methods lack an effective mechanism for balancing global and local information, relying on fixed or predefined weighting factors, which can be suboptimal for clients poorly represented by the global model. }

\textcolor{black}{To address these critical limitations, we propose a novel personalized federated LoRA fine-tuning approach called FedALT (short for Federated Fine-tuning through \textbf{A}daptive \textbf{L}ocal \textbf{T}raining with Rest-of-World LoRA.). Distinctly different from FedAvg-based methods, FedALT does not aggregate local models to initialize training rounds. Instead, each client continues training on it own previously trained local models, supplemented by a frozen \textbf{Rest-of-World (RoW) LoRA} component. The RoW LoRA captures global knowledge without updating during local training rounds, thus avoiding interference with client-specific adaptations. To leverage RoW LoRA effectively, we introduce an \textbf{adaptive mixer}, inspired by Mixture-of-Experts (MoE) \cite{jordan1994hierarchical}, which dynamically and optimally combines global and local information for each client's data sample. Our contributions are as follows:
\begin{itemize}
\item We highlight that personalized federated fine-tuning for clients with \textbf{heterogeneous tasks} has been insufficiently explored and identify critical limitations in existing methods that build upon the FedAvg paradigm.
\item We propose FedALT, a novel personalization framework featuring a frozen RoW LoRA and an adaptive mixer, effectively addressing these limitations by dynamically balancing global and local model adaptations.
\item We perform comprehensive experiments on two LLMs (Bloom \cite{le2023bloom} and Llama 2 \cite{touvron2023llama2}) using the Flan benchmark \cite{chung2024scaling}, demonstrating that FedALT significantly outperforms existing federated LoRA fine-tuning methods.
\end{itemize}
}

\section{Related Work}
\subsection{Parameter-Efficient Fine-Tuning}
As the size of LLMs continues to grow, full parameter fine-tuning has become increasingly computationally and storage-intensive \cite{hadi2023survey}. To address this issue, Parameter-Efficient Fine-Tuning (PEFT) methods \cite{ding2023parameter, fu2023effectiveness, han2024parameter, liu2022few, li2021prefix, lester2021power} have been developed to significantly reduce the number of trainable parameters. PEFT techniques introduce a limited set of additional trainable parameters to enhance model performance while keeping the majority of pre-trained parameters frozen. Among all existing PEFT methods, the most prominent approach is LoRA \cite{hu2021lora}, which employs low-rank matrices to approximate the pre-trained weight matrix, updating only the low-rank components. LoRA has became a standard method for adapting LLMs like Llama under limited computational resources \cite{lermen2023lora}. Several works have been proposed to improve LoRA \cite{tian2024hydralora, lermen2023lora, wang2023multilora, sheng2023s, liu2023moelora}. A recent development, Hydra-LoRA \cite{tian2024hydralora}, employs an asymmetric structure with a shared matrix for all samples and distinct matrices for each intrinsic component, thereby enhancing domain adaptation. While these advancements primarily focus on centralized learning settings, our work extends the application of PEFT methods to federated LLM fine-tuning, addressing the unique challenges posed by distributed and heterogeneous client data.

\subsection{Federated Learning}
The foundational work, FedAvg \cite{mcmahan2017communication}, addresses privacy and communication efficiency by aggregating local model to train a shared global model. Since then, numerous studies \cite{liuimproving, wang2024aggregation, yang2024fedfed, liu2025fedadamwcommunicationefficientoptimizerconvergence, bian2024accelerating, zhang2025fedel, 10.1145/3746027.3755226, liu2024fedbcgd} have focused on tackling various challenges within FL settings. One of the key challenges is data heterogeneity, which makes it difficult to train a single shared global model that performs well across all clients. To address this, Personalized Federated Learning (PFL) \cite{tan2022towards, deng2020adaptive, collins2021exploiting, fallah2020personalized} has emerged as an approach to adapt the global model to the specific needs of each client. Most PFL research has focused on addressing statistical differences in data distributions (e.g., label distribution skew) and has been applied primarily to smaller, standard models. In contrast, our work addresses the unique challenges posed by heterogeneous clients with diverse tasks by leveraging LLM fine-tuning within the FL framework.

\subsection{Federated Fine-Tuning}
Several studies \cite{cho2024heterogeneous, wang2024flora, kuang2024federatedscope, wu2024fedbiot} have explored federated fine-tuning approaches. \cite{kuang2024federatedscope} proposed federated full parameters fine-tuning, while \cite{sun2022conquering} introduced federated fine-tuning with PEFT using prefix-tuning. \cite{zhang2024towards} was the first study to apply LoRA in a federated context. \cite{bai2024federated, cho2024heterogeneous} focused on federated fine-tuning where clients have different LoRA ranks, while \cite{wang2024flora} proposed strategies for LoRA initialization to achieve better performance. Another studies \cite{Bian_2025_ICCV, sun2024improving} addressed server aggregation bias in LoRA-based federated fine-tuning.

The most relevant works to ours focus on personalized federated fine-tuning with LoRA \cite{yang2024dual, qi2024fdlora}. FedDPA \cite{yang2024dual} introduced global and local LoRA components but sufferd from interference during global LoRA training and struggled to balance the two phases effectively. PF2LoRA \cite{hao2025personalized} employed a bilevel framework to combine a shared adapter with client-specific adapters but faced interference issues during the aggregation of the shared adapter. FDLoRA \cite{qi2024fdlora} utilized personalized LoRA modules to initialize federated training, combining them with a global LoRA module via adaptive fusion; however, it relied on server-side datasets and was prone to interference during global LoRA aggregation, limiting its effectiveness in heterogeneous settings.

\section{Preliminaries and Motivations}
\subsection{Low-Rank Adaptation (LoRA)}
Consider a pre-trained model with parameters $\mathbf{W}_0 \in \mathbb{R}^{l \times d}$, where $\mathbf{W}_0$ represents the fixed parameters of the model, and $\mathbf{\Delta W} \in \mathbb{R}^{l \times d}$ denotes the trainable update matrix applied during fine-tuning. Here, $d$ is the input dimension and $l$ is the output dimension. Instead of updating all elements in $\mathbf{\Delta W}$, LoRA decomposes $\mathbf{\Delta W}$ into two low-rank matrices $\mathbf{A} \in \mathbb{R}^{r \times d}$ and $\mathbf{B} \in \mathbb{R}^{l \times r}$, where $r \ll \min(d, l)$. This decomposition allows the fine-tuning process to focus on the significantly smaller low-rank matrices $\mathbf{A}$ and $\mathbf{B}$ instead of the full matrix $\mathbf{\Delta W}$. Consequently, the total number of trainable parameters is reduced from $d \times l$ to $r \times (d + l)$. The model parameters after fine-tuning are given by:
\begin{equation}
\mathbf{W} = \mathbf{W}_0 + \mathbf{\Delta W} = \mathbf{W}_0 + \mathbf{B} \mathbf{A},
\end{equation}
where $\mathbf{A}$ is typically initialized with random Gaussian values, while $\mathbf{B}$ is initialized to zero.


\subsection{Heterogeneous Federated Fine-Tuning}
Assume there are $K$ institutions (clients), each possessing a distinct local training dataset $\mathcal{D}_k = \{X_k, Y_k\}$, where $k$ indexes a client. In the context of LLM fine-tuning, it is common for clients to have datasets corresponding to entirely different tasks (e.g., one client may have data for text summarization, while another client has data for sentiment analysis), rather than simply exhibiting statistical differences as in conventional FL. 

In this scenario, each client $k$ aims to leverage FL to extract useful information while fine-tuning a model tailored to its specific task. The goal is for each client to learn an individualized fine-tuned model $\textbf{W}_k$ that best fits its local dataset. The objective can be formulated as:
\begin{equation}
    \min_{\mathcal{W}} \frac{1}{K} \sum_{k=1}^K L_k(X_k, Y_k, \mathbf{W}_k),
\end{equation}
where $L_k(\cdot)$ represents the loss function for client $k$, and $\mathcal{W}$ denotes the set of personalized models $\{\mathbf{W}_k\}_{k=1}^K$.

\subsection{Motivational Study}

Federated fine-tuning operates in highly \textcolor{black}{heterogeneous environments} where client tasks can differ significantly. To understand the effects of such heterogeneity, we conducted a motivational study that highlights both the benefits and challenges of applying FL in these settings.

\noindent\textbf{Fact 1:} \textit{FL Can Transfer Useful Knowledge but Also Introduces Harmful Interference.}

In FL settings, the data collected by each client can be highly diverse \cite{huang2022learn, mendieta2022local, wang2024taming}. We conducted a motivational experiment using the Flan dataset \cite{chung2024scaling}, selecting 8 tasks and creating 8 clients, each specializing in a single task. To demonstrate potential interference in federated fine-tuning, we compared two baseline methods. The first method, \textit{FedIT} \cite{zhang2024towards}, combines FedAvg with LoRA fine-tuning. The second method is \textit{Local Only} LoRA fine-tuning, where each client fine-tunes its LoRA locally without any communication or information sharing between clients. 
The results, shown in \cref{fig:motivation}, support our first observation. On one hand, federated fine-tuning using FedIT achieves better performance than local fine-tuning on tasks Coreference Resolution and Paraphrase Detection, confirming that FL can enhance learning by enabling clients to benefit from shared knowledge. On the other hand, for tasks Commonsense Reasoning and Text Classification, FedIT underperforms compared to individual local fine-tuning. This dual outcome underscores both the promise and the pitfalls of federated fine-tuning in heterogeneous settings.

\begin{figure}
  \centering
    \includegraphics[width=0.85\linewidth]{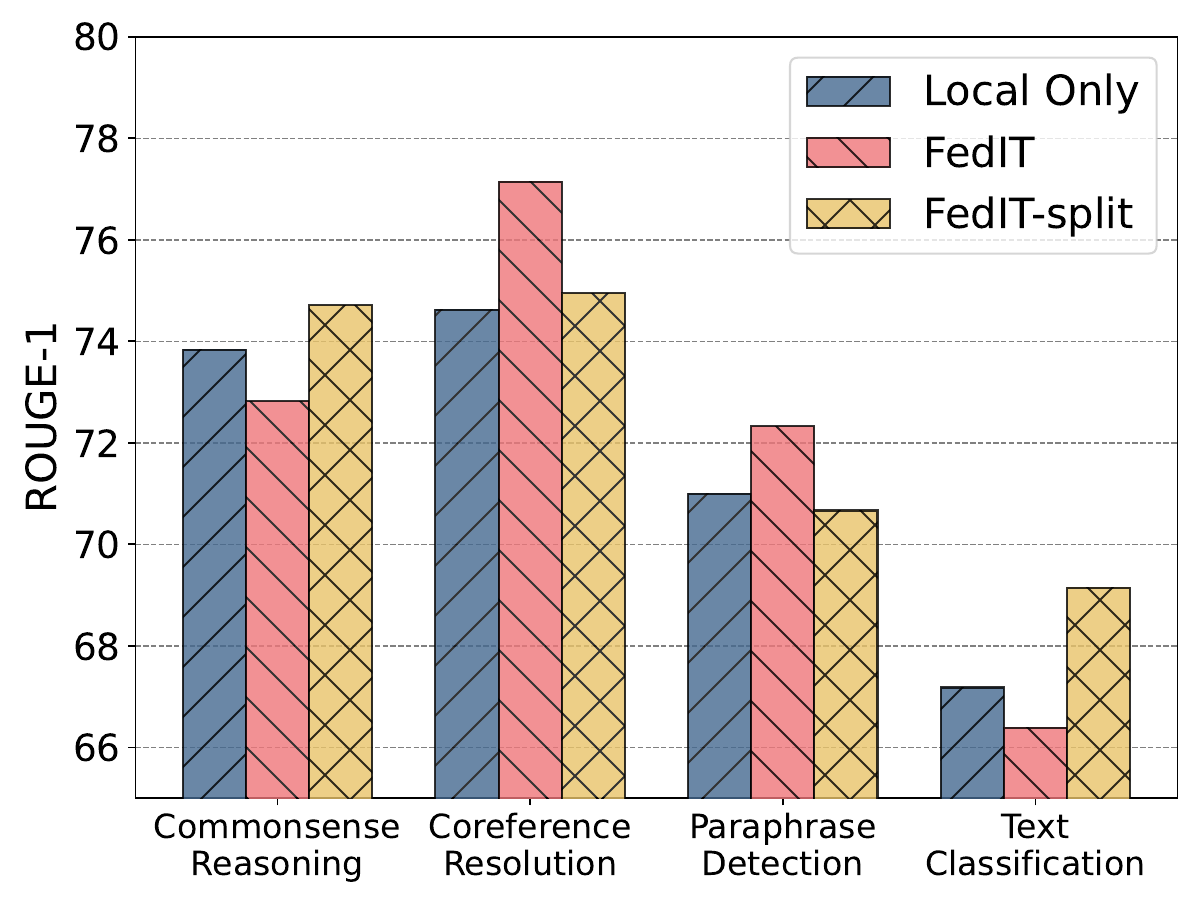}
  \caption{\textbf{Motivational study results}. }
  \label{fig:motivation}
\end{figure}

\noindent\textbf{Fact 2:} \textit{FedAvg Paradigm Struggles to Mitigate Interference}

Having established the existence of harmful interference, we next explored potential mitigation strategies. In centralized LoRA fine-tuning, HydraLoRA \cite{tian2024hydralora} found that splitting a single large LoRA into multiple smaller LoRAs reduces cross-domain interference. However, FL introduces additional challenges, as data domains are distributed across clients. To test whether multiple smaller LoRAs could mitigate interference in FL, we implemented FedAvg with multiple parallel LoRAs (i.e. FedIT-split). However, as shown in \cref{fig:motivation}, its performance remained comparable to FedIT (FedAvg with a single large LoRA), indicating that interference persisted. This result is expected: in centralized settings, interference occurs within a shared dataset and can be mitigated by partitioning LoRAs. In FL, interference arises primarily during server aggregation, and simply introducing multiple LoRAs does not prevent the aggregated model from disrupting local adaptations. Consequently, direct adaptations of centralized interference-mitigation strategies fail to address the core issue in FL.

\section{Proposed Method: FedALT}
\label{sec: proposed_method}
Our motivational studies reveal that harmful interference in FL arises from server aggregation and the use of the aggregated model to initialize subsequent training rounds. In the FedAvg paradigm, a global LoRA serves as a shared initialization point, allowing clients to fine-tune locally on their respective datasets. After fine-tuning, the updated LoRAs are averaged on the server and redistributed to clients as the new initialization. However, when client data differ significantly, aggregation disrupts individual adaptations, often negating improvements from local training and hindering personalization. Can cross-client interference be mitigated? The simplest and most direct solution is local fine-tuning using only a client’s own dataset. This approach eliminates interference entirely but fails to leverage shared knowledge from other clients. To address this trade-off, we propose a novel personalized federated LoRA fine-tuning algorithm that diverges from the FedAvg paradigm, prioritizing local training while selectively integrating useful global information.  

Our method introduces two key innovations:  1. Each client maintains two LoRA modules: an \textbf{Individual LoRA}, which captures locally learned information and is updated exclusively during the client's fine-tuning, and a \textbf{Rest-of-World (RoW) LoRA}, which aggregates information from all other clients. Crucially, the RoW LoRA remains fixed during local fine-tuning, ensuring that aggregation does not interfere with a client's adaptation.  2. An \textbf{adaptive mixer} dynamically learns the optimal \textbf{input-specific} weighting between the Individual LoRA and RoW LoRA components to balance personalization and global knowledge integration. In the following of this section, we detail the design and functionality of these components. The theoretical convergence analysis is provided in the supplementary material.

\subsection{Individual and Rest-of-World (RoW) LoRA}  
Each client \( k \) maintains two types of LoRAs: an Individual LoRA, denoted by \( \mathbf{A}^L_k/\mathbf{B}^L_k \), and a RoW LoRA, denoted by \( \mathbf{A}^R_k/\mathbf{B}^R_k \). The Individual LoRA captures information obtained locally by the client and is isolated from the information of other clients. In contrast, the RoW LoRA represents the average of the Individual LoRAs from all other clients and remains frozen during the client’s local fine-tuning. Specifically, at the end of each round, client \( k \) updates its Individual LoRA \( \mathbf{A}^L_k/\mathbf{B}^L_k \) through local training. The RoW LoRA for client \( k \) is then computed as:
\begin{align}
\label{Rotw_compute}
    \mathbf{A}^R_k = \frac{1}{K-1} \sum_{m \neq k} \mathbf{A}^L_m, \quad \mathbf{B}^R_k = \frac{1}{K-1} \sum_{m \neq k} \mathbf{B}^L_m,
\end{align}
After clients upload their local LoRAs to the server, the server computes the RoW LoRA for each client individually and distributes it back to them. Alternatively, the server can calculate the global average of all clients' individual LoRAs and distribute this average back to the clients. Each client then determines its RoW LoRA by subtracting its individual LoRA from the global average LoRA.

It is crucial to emphasize that the RoW LoRA is not further updated on clients using their local datasets. Instead, it is only updated through averaging at the end of each learning round. Training the RoW LoRA on local datasets is unnecessary and potentially counterproductive, as subsequent averaging could cancel out any improvements. Furthermore, skipping local training for the RoW LoRA reduces the computational workload on clients by half.

\subsection{Dynamic Mixture-of-Experts Mechanism}  
A natural question arises: why not simply add the RoW LoRA directly to the pre-trained model, i.e., \( \mathbf{W}_0 \leftarrow \mathbf{W}_0 + \mathbf{B}^R_k \mathbf{A}^R_k \), since the RoW LoRA remains fixed during local training? This approach would allow each client to maintain only a single Individual LoRA, saving memory space. However, there are two key reasons why this is not desirable:  
\begin{enumerate}
    \item \textbf{Potential Contamination of the Pre-trained Model}: If the RoW LoRA underperforms, adding it directly to the pre-trained model risks contaminating it, making further corrections difficult.  
    \item \textbf{Loss of Flexibility}: Directly incorporating the RoW LoRA enforces a fixed weight, preventing dynamic adjustment between the Individual LoRA and shared information for different inputs. This lack of flexibility is problematic because different inputs may benefit variably from the local model versus the globally averaged model. A ``one-size-fits-all'' approach often fails to perform optimally across diverse inputs.  
\end{enumerate}
To address these challenges, we propose a dynamic weighting mechanism for combining the Individual LoRA and the RoW LoRA, enabling input-specific flexibility. Specifically, we introduce a mixer to dynamically adjust the contributions of the two LoRAs for each input. During the forward pass of each local epoch, the model's output is computed as:  
\[
y = \mathbf{W}_0 x + \alpha_k(x) \mathbf{B}^L_k \mathbf{A}^L_k x + \big(1 - \alpha_k(x)\big) \mathbf{B}^R_k \mathbf{A}^R_k x,  
\]  
where \( \alpha_k(x) \) determines for client $k$ the contribution weight of the Individual LoRA, and \( 1 - \alpha_k(x) \) determines the contribution of the RoW LoRA.  

Given the complexity of modern LLMs, where LoRA modules are inserted at multiple layers, learning \( \alpha(x) \) is non-trivial. Inspired by \cite{jordan1994hierarchical}, we adopt a Mixture-of-Experts (MoE) approach to compute \( \alpha(x) \). Specifically, we introduce a mixer \( \mathbf{G}_k \in \mathbb{R}^{2 \times d} \) which is a dense layer with trainable weights (a transformation matrix), followed by a softmax function, to learn the weighting:  
\begin{align}
\alpha(x), \, 1 - \alpha(x) = \text{softmax}(\mathbf{G}_kx).    
\end{align}
Importantly, the MoE mixers are personalized for each client, tailored to their specific domain tasks, and are not averaged across clients. This personalization ensures that the weight adjustments reflect each client's unique data distribution and learning objectives.  

\subsection{Algorithm Workflow}
By combining the two LoRA modules (Individual LoRA and RoW LoRA) with dynamic weights, our proposed method achieves a balance between leveraging useful information from other clients and isolating harmful interference. To better illustrate the FedALT training process, we provide an overview of the server and client operations at round \( t \).

\textbf{Server Side.}  
After local fine-tuning in round \( t \), the server receives each client \( k \)'s Individual LoRA modules (\( \mathbf{A}^L_k/\mathbf{B}^L_k \)). Upon receiving all clients' Individual LoRA modules, the server updates and computes the RoW LoRA modules for each client \( k \) based on \cref{Rotw_compute}. After calculating the RoW LoRA modules for client \( k \), the server broadcasts these updated RoW LoRA to the respective client.

\textbf{Client Side.}  
Each client \( k \) replaces its current RoW LoRA modules with the new RoW LoRA modules (\( \mathbf{A}^R_k/\mathbf{B}^R_k \)) received from the server. It then begins local fine-tuning for round \( t+1 \) using its local dataset. During the local fine-tuning process, the client updates the parameters in its Individual LoRA modules (\( \mathbf{A}^L_k/\mathbf{B}^L_k \)) and the mixer \( \mathbf{G}_k \), while the RoW LoRA modules (\( \mathbf{A}^R_k/\mathbf{B}^R_k \)) and the pre-trained model (\( \mathbf{W}_0 \)) remain fixed. Once the local fine-tuning is complete, the updated Individual LoRA is uploaded to the server, while \( \mathbf{G}_k \) stays local to the client.

\begin{table*}[t!]
\centering

\resizebox{0.95\textwidth}{!}{%
\begin{tabular}{l|cccccccc|c}
\toprule
\multirow{2}{*}{\textbf{Methods}} & \textbf{Commonsense} & \textbf{Coreference} & \textbf{Natural Language} & \textbf{Paraphrase} & \textbf{Reading} & \textbf{Sentiment} & \textbf{Struct to} & \textbf{Text} & \multirow{2}{*}{\textbf{Average}} \\
& \textbf{Reasoning} & \textbf{Resolution} & \textbf{Inference} & \textbf{Detection} & \textbf{Comprehension} & \textbf{Analysis} & \textbf{Text} & \textbf{Classification} & \\
\hline
\multicolumn{10}{c}{\textbf{LLaMA2-7B}} \\
\hline
Local  & 73.83 & 74.62 & 58.73 & 71.00 & 55.82 & 46.56 & \textbf{55.14} & 67.18 & 62.86 \\
FedIT       & 72.82 & 77.14 & 59.53 & 72.33 & 53.35 & 46.76 & 49.23 & 66.39 & 62.19 \\
FFA-LoRA    & 66.96 & 68.48 & 50.41 & 69.79 & 49.69 & 44.44 & 43.29 & 61.15 & 56.78 \\
FedSA       & 72.68 & 78.24 & 64.46 & 76.33 & 54.32 & 42.71 & 53.57 & 65.47 & 63.47 \\
FedDPA      & 74.81 & 81.88 & 62.92 & 76.33 & 55.91 & 47.86 & 52.02 & 65.42 & 64.64 \\
PF2LoRA       & 74.13 & 77.55 & 64.17 & \textbf{78.33} & 55.36 & 48.65 & 53.44 & 63.90 & 64.44 \\
FDLoRA       & \textbf{76.29} & 75.60 & 66.12 & 75.67 & 57.39 & 49.85 & 52.85 & 67.59 & 65.17 \\
\rowcolor{gray!40} \textbf{FedALT}      & 76.12 & \textbf{83.04} & \textbf{67.73} & 77.80 & \textbf{59.41} & \textbf{51.57} & 53.12 & \textbf{71.60} & \textbf{67.55} \\
\hline
\multicolumn{10}{c}{\textbf{Bloom-560M}} \\
\hline
Local      & 52.13 & 39.20 & 37.55 & 62.67 & 29.45 & 42.91 & \textbf{41.33} & 49.99 & 44.40 \\
FedIT      & 50.24 & 40.08 & 38.51 & 67.00 & 27.98 & 40.34 & 38.95 & 51.52 & 44.33 \\
FFA-LoRA   & 46.07 & 35.05 & 33.57 & 66.12 & 25.52 & 37.13 & 35.43 & 44.63 & 40.44 \\
FedSA      & 55.79 & 41.65 & 41.12 & 69.33 & 27.90 & \textbf{43.72} & 41.27 & 55.03 & 46.98 \\
FedDPA     & 55.43 & \textbf{41.88} & 41.80 & 69.80 & 29.74 & 41.89 & 39.54 & 55.85 & 46.99 \\
PF2LoRA    & 54.58 & 40.16 & 37.40 & 70.08 & 25.78 & 42.36 & 37.90 & 58.67 & 45.87 \\
FDLoRA     & 53.65 & 40.22 & 37.27 & 68.89 & 29.95 & 42.89 & 35.39 & 56.05 & 45.54 \\
\rowcolor{gray!40} \textbf{FedALT}     & \textbf{56.39} & 41.27 & \textbf{44.45} & \textbf{70.67} & \textbf{30.62} & 43.04 & 39.24 & \textbf{59.03} & \textbf{48.09} \\
\bottomrule
\end{tabular}%
}
\caption{\textbf{Performance comparison} with baseline methods across different models.}
\label{tab:performance_comparison_merged}
\end{table*}

\section{Experiments}

\textcolor{black}{\textcolor{black}{\textbf{Datasets.} Our training datasets are derived from Flan \cite{chung2024scaling}, a comprehensive benchmark comprising over 60 NLP datasets across more than 12 diverse task types. To simulate client heterogeneity in the main experiments, we assume 8 clients, each assigned a distinct NLP task dataset. Detailed descriptions, including task distributions and client assignments, are provided in supplementary material.}}

\textbf{Pre-trained Model and Baselines.}  
We utilize \textcolor{black}{LLaMA2-7B \cite{touvron2023llama2} and Bloom-560M \cite{le2023bloom}} as the pre-trained models for all baseline and proposed fine-tuning methods. To evaluate the performance of our proposed method, we compare it with the state-of-the-art methods in federated fine-tuning with LoRA, including \textbf{FedIT} \cite{zhang2024towards}, \textbf{FFA-LoRA} \cite{sun2024improving}. Additionally, we compare it with state-of-the-art methods focused on personalized federated fine-tuning with LoRA, such as \textbf{FedDPA} \cite{yang2024dual}, \textbf{PF2LoRA} \cite{hao2025personalized}, \textbf{FedSA} \cite{guo2024selective}, and \textbf{FDLoRA} \cite{qi2024fdlora}. We also include the \textbf{Local Only} baseline, where clients fine-tune their models independently without federated collaboration. Detailed descriptions of these baseline methods are provided in the supplementary material.

\begin{figure}[h]
  \centering
  \resizebox{0.95\linewidth}{!}{ 
    \begin{minipage}{\linewidth}
      \begin{subfigure}[t]{\linewidth}
        \centering
        \includegraphics[width=0.9\linewidth]{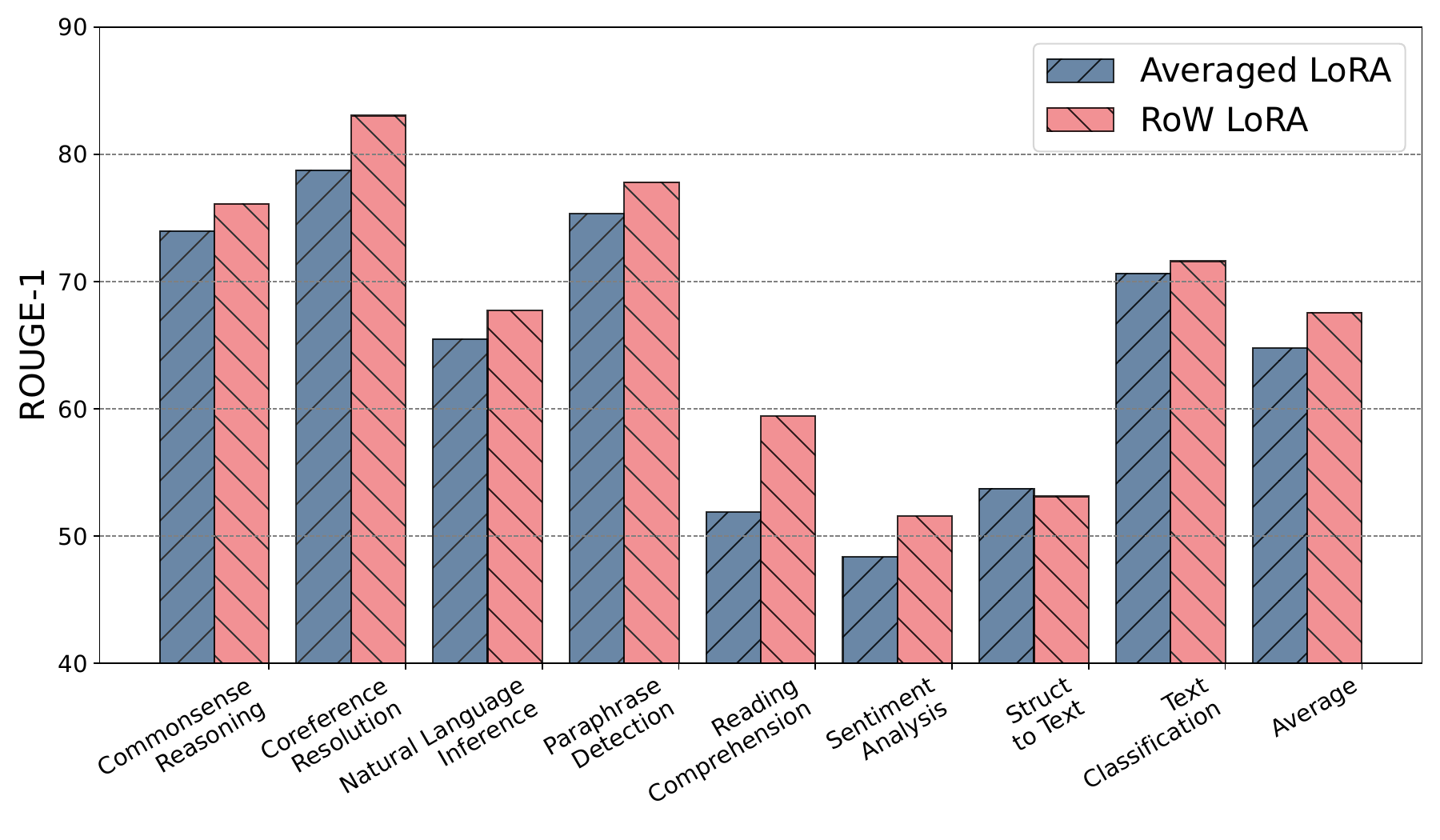}
        \caption{\textbf{Impact of RoW LoRA.} Isolating client-specific information is more effective than using globally averaged updates. }
        \label{fig:impact_rotw}
      \end{subfigure}
      
      \begin{subfigure}[t]{\linewidth}
        \centering
        \includegraphics[width=0.9\linewidth]{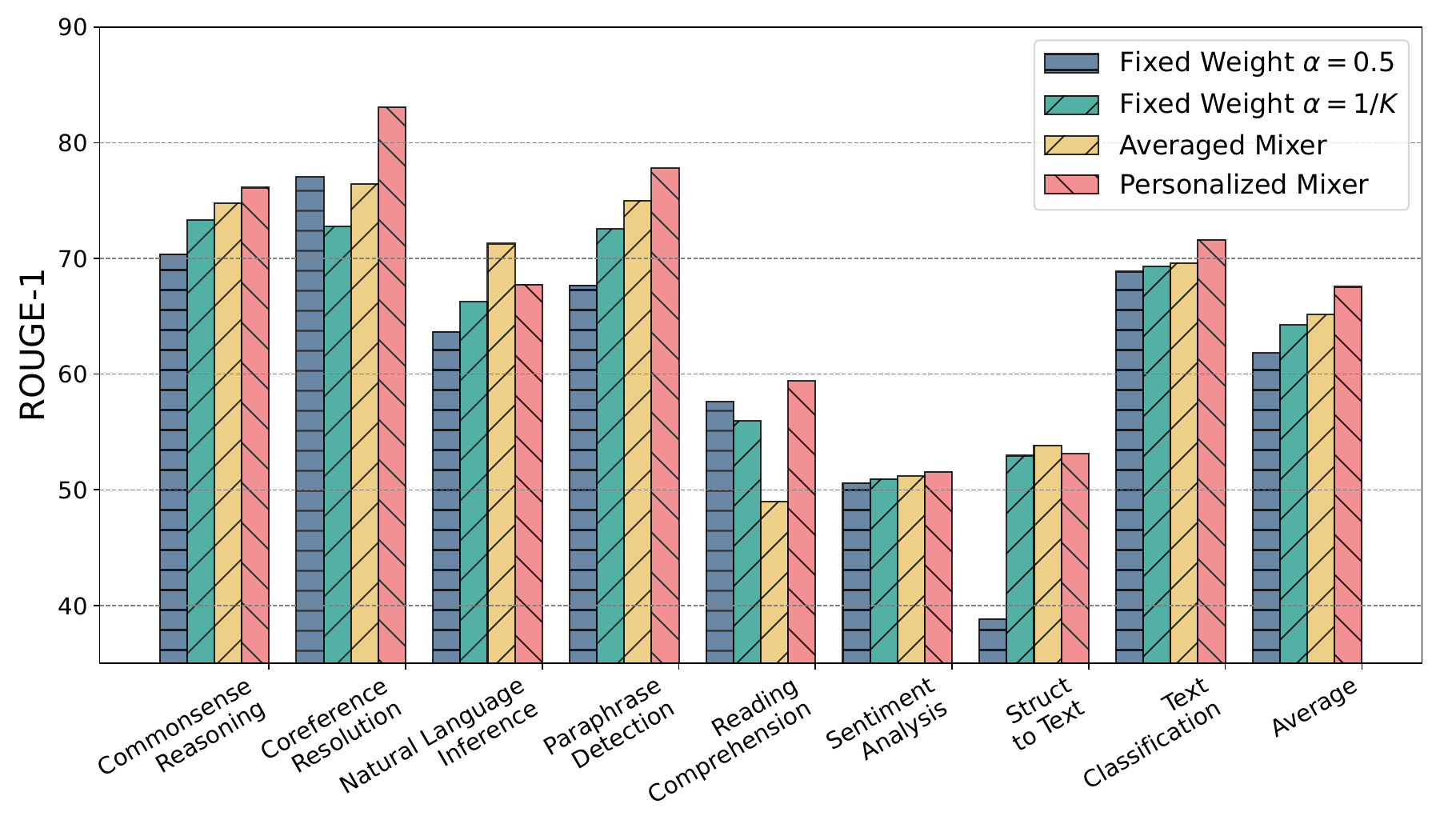}
        \caption{\textbf{Impact of dynamic weighting via MOE.} The personalized mixer outperforms fixed or shared weighting.}
        \label{fig:impact_moe}
      \end{subfigure}
    \end{minipage}
  }
  \caption{Ablation studies of FedALT.}
  \label{fig:impact_combined}
\end{figure}

\begin{figure*}[htbp]
  \centering

  \begin{subfigure}[t]{0.31\linewidth}
    \centering
    \includegraphics[width=\linewidth]{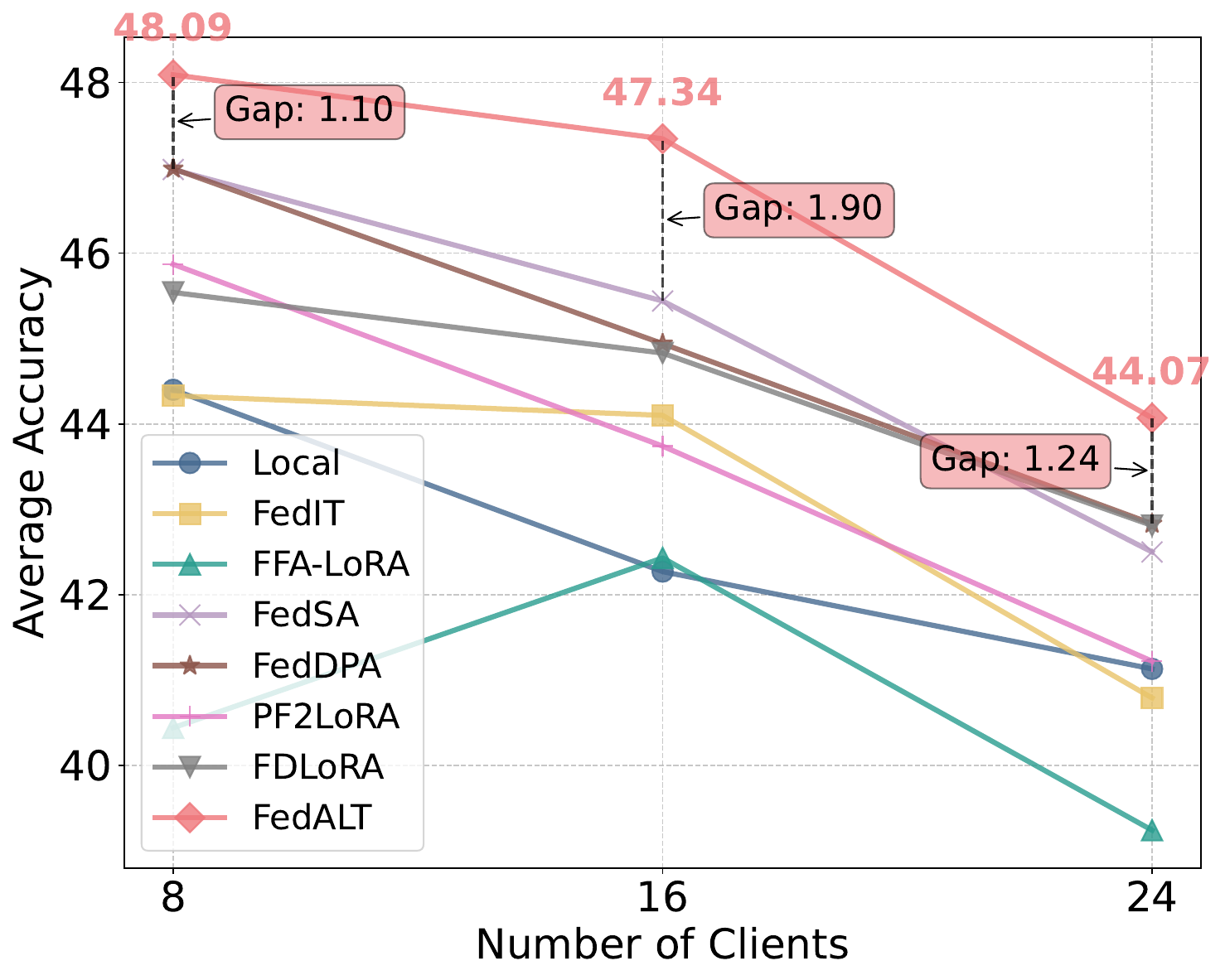}
    \caption{\textbf{Effect of number of clients.}}
    \label{fig:ablation_clients}
  \end{subfigure}
  \hfill
  \begin{subfigure}[t]{0.31\linewidth}
    \centering
    \includegraphics[width=\linewidth]{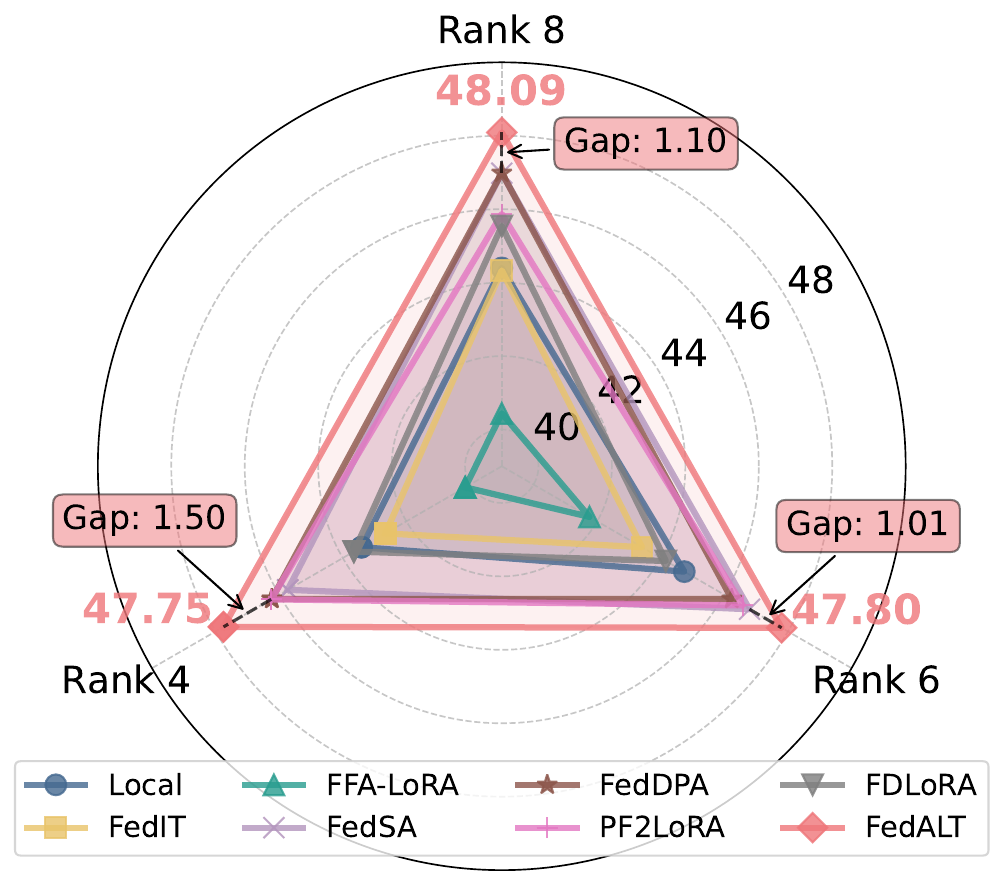}
    \caption{\textbf{Effect of LoRA rank.} }
    \label{fig:ablation_rank}
  \end{subfigure}
  \hfill
  \begin{subfigure}[t]{0.31\linewidth}
    \centering
    \includegraphics[width=\linewidth]{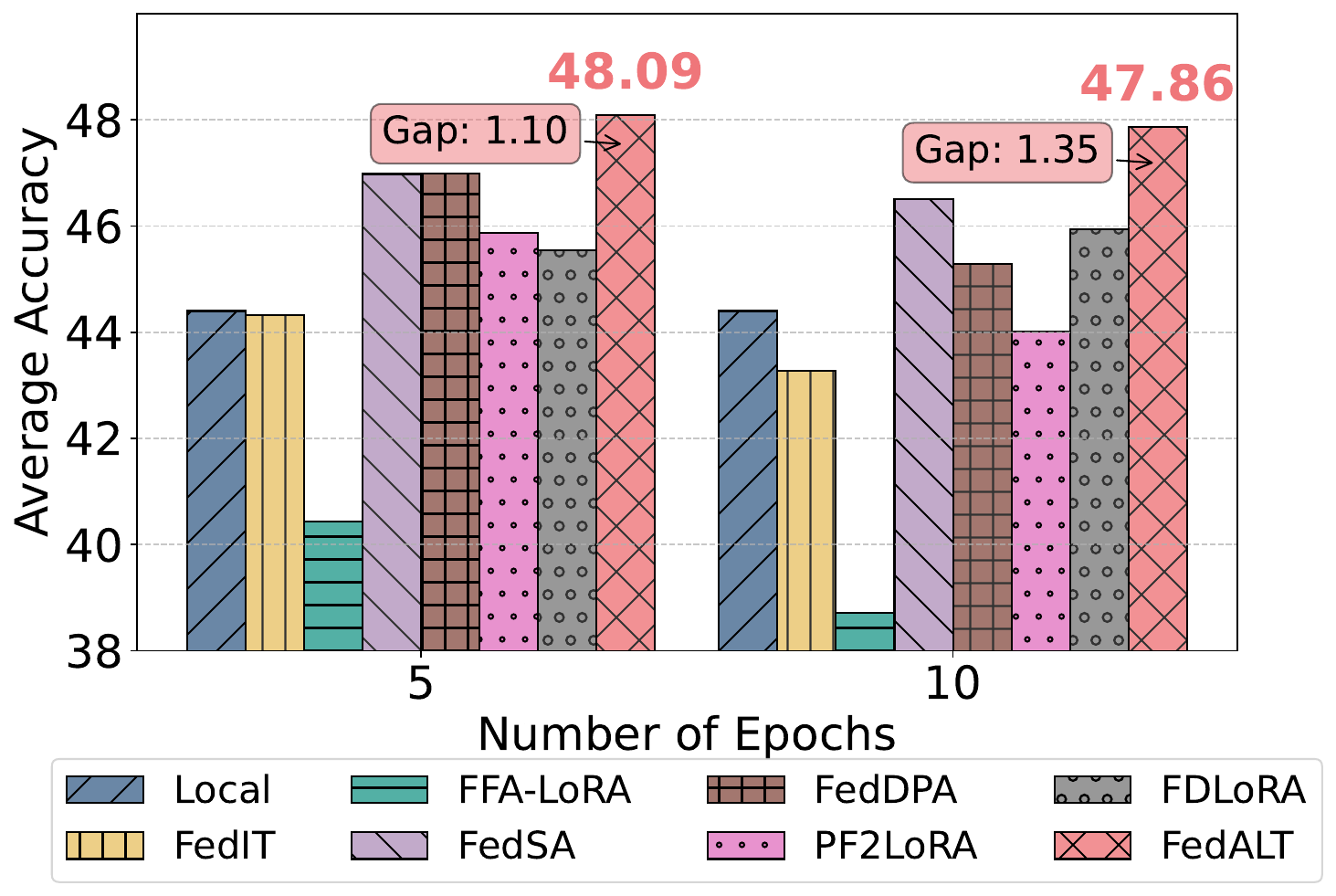}
    \caption{\textbf{Effect of local epochs.} }
    \label{fig:ablation_epochs}
  \end{subfigure}

  \caption{Sensitivity Analyses of FedALT under different training configurations.}
  \label{fig:ablation_combined}
\end{figure*}

\subsection{Performance Comparisons}
\label{sec :exp_res}
We compare the performance of FedALT, with other baseline methods based on each client’s performance on its specific test dataset, aligned with its training dataset. As shown in \cref{tab:performance_comparison_merged}, which reports the results on LLaMA2-7B pretrained models, our proposed method achieves the best overall performance compared to all baselines, with the average score exceeding that of the best-performing baseline, FDLoRA, by 2.38. FedALT consistently delivers superior performance across most tasks, outperforming the baselines by at least 1.61 on the Natural Language Inference task and 2.02 on the Reading Comprehension task. We find that on the Struct-to-Text, the Local Only method achieves the highest performance. This is likely due to the unique nature of this task, making it difficult to benefit from shared global knowledge. Furthermore, we observe that FedIT underperforms compared to Local Only on four client tasks, indicating the presence of harmful interference during aggregation. Personalized federated LoRA fine-tuning methods also fall short compared to FedALT, likely because they fail to effectively isolate harmful interference during the server aggregation step. Despite their designs for local/global LoRA interaction, these methods update LoRA modules using aggregated global information, which may carry conflicting information from other clients.

\textcolor{black}{To further evaluate the robustness of our proposed FedALT, we conduct experiments using the Bloom-560M model. As shown in \cref{tab:performance_comparison_merged}, FedALT consistently outperforms all baseline methods across the NLP tasks. These empirical results validate that conventional FedAvg-based approaches suffer from harmful cross-client parameter interference during aggregation. In contrast, FedALT's adaptive strategy, which balances local specialization and global knowledge, effectively alleviates this issue. The consistent performance gains observed with both LLaMA2-7B and Bloom-560M confirm the generalizability of our approach across diverse model families and parameter scales.}

\subsection{Ablation Study}
\label{ablation}

\begin{figure}
\centering
\begin{adjustbox}{width=0.9\linewidth,center}
\begin{tabular}{c|c|c|c}
\toprule
\textbf{Methods} & \textbf{Same Trainable} & \textbf{Same Inference} & \textbf{Avg. Perf.} \\
\midrule
RoW-Update (rank = 4) & $\checkmark$ & $\times$ & 63.59 \\
RoW-Update (rank = 8) & $\times$ & $\checkmark$ & 65.75 \\
\textbf{FedALT} & -- & -- & \textbf{67.55} \\
\bottomrule
\end{tabular}
\end{adjustbox}
\captionof{table}{\textbf{Impact of decoupling LoRA training.} }
\label{tab:decoupling}
\end{figure}

\textbf{Impact of Decoupling LoRA Training.} In our proposed method, only the Individual LoRA modules are updated locally during client fine-tuning, while the RoW LoRA remains frozen and is updated exclusively through server-side aggregation. This design mitigates harmful cross-client interference by preventing local updates to the shared global component. To evaluate the effectiveness of this decoupled training strategy, we compare FedALT with two alternative methods. The first baseline, \textit{RoW-Update (rank = 4)}, fine-tunes both the Individual and RoW LoRA modules locally, each with a rank of 4 to maintain the same total number of trainable parameters. The second baseline, \textit{RoW-Update (rank = 8)}, also fine-tunes both LoRA modules locally, but assigns a rank of 8 to each, resulting in a higher overall parameter count. For fair comparison, both baselines use the dynamic mixer. The key distinction lies in whether the RoW LoRA is updated locally. As shown in \cref{tab:decoupling}, while the second method outperforms the first due to its increased capacity, both alternatives underperform to FedALT. Notably, even with more trainable parameters, the second method fails to match FedALT’s performance, highlighting the benefits of decoupling local updates from the global component.

\textbf{Impact of RoW LoRA}
To evaluate the impact of using the RoW LoRA instead of the global average of all Local LoRAs (which includes each client’s Individual LoRA), we conduct an experiment. As shown in \cref{fig:impact_rotw}, leveraging the RoW LoRA significantly improves performance compared to using the global average. This result underscores the importance of isolating client-specific information from the aggregated global knowledge.

\textbf{Impact of Dynamic Weighting via MOE.} 
To enable input-specific flexibility, our proposed method adopts a MOE mechanism to dynamically adjust the contributions of the Individual LoRA and RoW LoRA \textcolor{black}{for each data sample}. We conduct experiments to evaluate the efficiency of this dynamic weighting by comparing it to two fixed weighting strategies. In the first fixed strategy, we set the contribution weight to a fixed value of \( \alpha = 0.5 \), meaning that the Individual LoRA and RoW LoRA contribute equally. In the second fixed strategy, we set the contribution weight to \( \alpha = 1/K \), where \( K \) is the number of clients. The results in \cref{fig:impact_moe} validate that our proposed dynamic weighting mechanism outperforms both fixed strategies, demonstrating its ability to enhance performance. Additionally, our dynamic weighting is implemented through the personal \( \mathbf{G}_k \) mixer, which is maintained locally by each client and is not aggregated during the training process. To further investigate this design choice, we conducted experiments where \( \mathbf{G}_k \) is averaged across clients. The results confirm that retaining a personalized \( \mathbf{G}_k \) for each client yields optimal performance, emphasizing the importance of input-specific and client-specific adjustments.


\subsection{Sensitivity Analysis}
\label{Sensitivity}


\textcolor{black}{\textbf{Effect of Number of Clients.} In our main experiments, we demonstrated the effectiveness of the proposed method using 8 clients. Here we extend our evaluation by conducting additional experiments with 16 and 24 clients, using BLOOM as the foundation model. For the 16-client setting, we distributed each original dataset from the 8-client configuration between two clients, resulting in each client having approximately half the data samples compared to the 8-client scenario, while maintaining the same 8 distinct task types across the federation. The results in \cref{fig:ablation_clients} show that FedALT consistently maintains its performance advantage even as the number of clients increases, outperforming the strongest baseline by 1.90 at least. Note that FedALT remains efficient due to its architectural design: each client maintains only its Individual LoRA and a single RoW LoRA, rather than storing separate components for every other client. Additionally, the server-side computation of the RoW LoRA remains lightweight, as it involves only simple averaging operations.}


\textcolor{black}{\textbf{Effect of LoRA Rank.} We evaluate the impact of different LoRA ranks \( r \in \{4, 6, 8\} \) using BLOOM as the pre-trained model, keeping all other settings consistent with the main experiment. As shown in \cref{fig:ablation_rank}, the proposed method consistently outperforms baseline methods across all rank values, demonstrating its adaptability and robustness.}

\textcolor{black}{\textbf{Effect of Local Epochs.} We investigate the effect of varying local training epochs (5 and 10), while keeping all other configurations unchanged. As presented in \cref{fig:ablation_epochs}, our method outperforms all baselines, further highlighting its robustness to different local training epochs.} 

Additional experiments on the \textbf{effect of data heterogeneity} can be found in the supplementary material.


\section{Conclusion}

In this work, we introduced FedALT, a personalized federated LoRA fine-tuning framework that departs from conventional FedAvg paradigms. Our key innovations include a decoupled training scheme with Individual and frozen Rest-of-World (RoW) LoRA components that mitigates harmful cross-client interference in heterogeneous settings. Additionally, we introduce an adaptive mixer inspired by Mixture-of-Experts that dynamically balances local and global information on an input-specific basis. Through comprehensive evaluation across diverse NLP tasks, FedALT demonstrates superior performance compared to state-of-the-art baselines while maintaining computational efficiency. Future research directions include investigating clustering-based approaches for maintaining multiple RoW LoRA components, and extending our framework to multi-modal foundation models.

\section{Acknowledgments}
The work of Jieming Bian, Lei Wang and Jie Xu is partially supported by NSF under grants 2433886, 2505381 and 2515982. The work of Letian Zhang is partially supported by NSF under grant 2348279 and also supported by MTSU Stark Land project. 

\bibliography{aaai2026}

@String(ICCV  = {Int. Conf. Comput. Vis.})

@String(AAAI  = {AAAI})

@String(ICASSP=	{ICASSP})

@String(ICCV  = {ICCV})

@inproceedings{kenton2019bert,
  title={Bert: Pre-training of deep bidirectional transformers for language understanding},
  author={Kenton, Jacob Devlin Ming-Wei Chang and Toutanova, Lee Kristina},
  booktitle={Proceedings of naacL-HLT},
  volume={1},
  number={2},
  year={2019},
  organization={Minneapolis, Minnesota}
}

@article{brown2020language,
  title={Language models are few-shot learners},
  author={Brown, Tom and Mann, Benjamin and Ryder, Nick and Subbiah, Melanie and Kaplan, Jared D and Dhariwal, Prafulla and Neelakantan, Arvind and Shyam, Pranav and Sastry, Girish and Askell, Amanda and others},
  journal={Advances in neural information processing systems},
  volume={33},
  pages={1877--1901},
  year={2020}
}

@article{raffel2020exploring,
  title={Exploring the limits of transfer learning with a unified text-to-text transformer},
  author={Raffel, Colin and Shazeer, Noam and Roberts, Adam and Lee, Katherine and Narang, Sharan and Matena, Michael and Zhou, Yanqi and Li, Wei and Liu, Peter J},
  journal={Journal of machine learning research},
  volume={21},
  number={140},
  pages={1--67},
  year={2020}
}

@article{touvron2023llama,
  title={Llama: Open and efficient foundation language models},
  author={Touvron, Hugo and Lavril, Thibaut and Izacard, Gautier and Martinet, Xavier and Lachaux, Marie-Anne and Lacroix, Timoth{\'e}e and Rozi{\`e}re, Baptiste and Goyal, Naman and Hambro, Eric and Azhar, Faisal and others},
  journal={arXiv preprint arXiv:2302.13971},
  year={2023}
}

@article{touvron2023llama2,
  title={Llama 2: Open foundation and fine-tuned chat models},
  author={Touvron, Hugo and Martin, Louis and Stone, Kevin and Albert, Peter and Almahairi, Amjad and Babaei, Yasmine and Bashlykov, Nikolay and Batra, Soumya and Bhargava, Prajjwal and Bhosale, Shruti and others},
  journal={arXiv preprint arXiv:2307.09288},
  year={2023}
}

@article{zeng2022glm,
  title={Glm-130b: An open bilingual pre-trained model},
  author={Zeng, Aohan and Liu, Xiao and Du, Zhengxiao and Wang, Zihan and Lai, Hanyu and Ding, Ming and Yang, Zhuoyi and Xu, Yifan and Zheng, Wendi and Xia, Xiao and others},
  journal={arXiv preprint arXiv:2210.02414},
  year={2022}
}

@article{zhou2024comprehensive,
  title={A comprehensive survey on pretrained foundation models: A history from bert to chatgpt},
  author={Zhou, Ce and Li, Qian and Li, Chen and Yu, Jun and Liu, Yixin and Wang, Guangjing and Zhang, Kai and Ji, Cheng and Yan, Qiben and He, Lifang and others},
  journal={International Journal of Machine Learning and Cybernetics},
  pages={1--65},
  year={2024},
  publisher={Springer}
}

@article{minaee2024large,
  title={Large language models: A survey},
  author={Minaee, Shervin and Mikolov, Tomas and Nikzad, Narjes and Chenaghlu, Meysam and Socher, Richard and Amatriain, Xavier and Gao, Jianfeng},
  journal={arXiv preprint arXiv:2402.06196},
  year={2024}
}

@article{hu2021lora,
  title={Lora: Low-rank adaptation of large language models},
  author={Hu, Edward J and Shen, Yelong and Wallis, Phillip and Allen-Zhu, Zeyuan and Li, Yuanzhi and Wang, Shean and Wang, Lu and Chen, Weizhu},
  journal={arXiv preprint arXiv:2106.09685},
  year={2021}
}

@article{ding2023parameter,
  title={Parameter-efficient fine-tuning of large-scale pre-trained language models},
  author={Ding, Ning and Qin, Yujia and Yang, Guang and Wei, Fuchao and Yang, Zonghan and Su, Yusheng and Hu, Shengding and Chen, Yulin and Chan, Chi-Min and Chen, Weize and others},
  journal={Nature Machine Intelligence},
  volume={5},
  number={3},
  pages={220--235},
  year={2023},
  publisher={Nature Publishing Group UK London}
}

@inproceedings{fu2023effectiveness,
  title={On the effectiveness of parameter-efficient fine-tuning},
  author={Fu, Zihao and Yang, Haoran and So, Anthony Man-Cho and Lam, Wai and Bing, Lidong and Collier, Nigel},
  booktitle={Proceedings of the AAAI conference on artificial intelligence},
  volume={37},
  number={11},
  pages={12799--12807},
  year={2023}
}

@article{han2024parameter,
  title={Parameter-efficient fine-tuning for large models: A comprehensive survey},
  author={Han, Zeyu and Gao, Chao and Liu, Jinyang and Zhang, Jeff and Zhang, Sai Qian},
  journal={arXiv preprint arXiv:2403.14608},
  year={2024}
}

@article{liu2022few,
  title={Few-shot parameter-efficient fine-tuning is better and cheaper than in-context learning},
  author={Liu, Haokun and Tam, Derek and Muqeeth, Mohammed and Mohta, Jay and Huang, Tenghao and Bansal, Mohit and Raffel, Colin A},
  journal={Advances in Neural Information Processing Systems},
  volume={35},
  pages={1950--1965},
  year={2022}
}

@article{li2021prefix,
  title={Prefix-tuning: Optimizing continuous prompts for generation},
  author={Li, Xiang Lisa and Liang, Percy},
  journal={arXiv preprint arXiv:2101.00190},
  year={2021}
}

@article{lester2021power,
  title={The power of scale for parameter-efficient prompt tuning},
  author={Lester, Brian and Al-Rfou, Rami and Constant, Noah},
  journal={arXiv preprint arXiv:2104.08691},
  year={2021}
}

@inproceedings{mcmahan2017communication,
  title={Communication-efficient learning of deep networks from decentralized data},
  author={McMahan, Brendan and Moore, Eider and Ramage, Daniel and Hampson, Seth and y Arcas, Blaise Aguera},
  booktitle={Artificial intelligence and statistics},
  pages={1273--1282},
  year={2017},
  organization={PMLR}
}

@inproceedings{zhang2024towards,
  title={Towards building the federatedGPT: Federated instruction tuning},
  author={Zhang, Jianyi and Vahidian, Saeed and Kuo, Martin and Li, Chunyuan and Zhang, Ruiyi and Yu, Tong and Wang, Guoyin and Chen, Yiran},
  booktitle={ICASSP 2024-2024 IEEE International Conference on Acoustics, Speech and Signal Processing (ICASSP)},
  pages={6915--6919},
  year={2024},
  organization={IEEE}
}

@article{sun2024improving,
  title={Improving loRA in privacy-preserving federated learning},
  author={Sun, Youbang and Li, Zitao and Li, Yaliang and Ding, Bolin},
  journal={arXiv preprint arXiv:2403.12313},
  year={2024}
}

@article{wang2024flora,
  title={Flora: Federated fine-tuning large language models with heterogeneous low-rank adaptations},
  author={Wang, Ziyao and Shen, Zheyu and He, Yexiao and Sun, Guoheng and Wang, Hongyi and Lyu, Lingjuan and Li, Ang},
  journal={arXiv preprint arXiv:2409.05976},
  year={2024}
}

@article{guo2024selective,
  title={Selective Aggregation for Low-Rank Adaptation in Federated Learning},
  author={Guo, Pengxin and Zeng, Shuang and Wang, Yanran and Fan, Huijie and Wang, Feifei and Qu, Liangqiong},
  journal={arXiv preprint arXiv:2410.01463},
  year={2024}
}

@article{yang2024dual,
  title={Dual-Personalizing Adapter for Federated Foundation Models},
  author={Yang, Yiyuan and Long, Guodong and Shen, Tao and Jiang, Jing and Blumenstein, Michael},
  journal={arXiv preprint arXiv:2403.19211},
  year={2024}
}

@article{qi2024fdlora,
  title={FDLoRA: Personalized Federated Learning of Large Language Model via Dual LoRA Tuning},
  author={Qi, Jiaxing and Luan, Zhongzhi and Huang, Shaohan and Fung, Carol and Yang, Hailong and Qian, Depei},
  journal={arXiv preprint arXiv:2406.07925},
  year={2024}
}

@article{tan2022towards,
  title={Towards personalized federated learning},
  author={Tan, Alysa Ziying and Yu, Han and Cui, Lizhen and Yang, Qiang},
  journal={IEEE transactions on neural networks and learning systems},
  volume={34},
  number={12},
  pages={9587--9603},
  year={2022},
  publisher={IEEE}
}

@article{deng2020adaptive,
  title={Adaptive personalized federated learning},
  author={Deng, Yuyang and Kamani, Mohammad Mahdi and Mahdavi, Mehrdad},
  journal={arXiv preprint arXiv:2003.13461},
  year={2020}
}

@inproceedings{collins2021exploiting,
  title={Exploiting shared representations for personalized federated learning},
  author={Collins, Liam and Hassani, Hamed and Mokhtari, Aryan and Shakkottai, Sanjay},
  booktitle={International conference on machine learning},
  pages={2089--2099},
  year={2021},
  organization={PMLR}
}

@article{fallah2020personalized,
  title={Personalized federated learning: A meta-learning approach},
  author={Fallah, Alireza and Mokhtari, Aryan and Ozdaglar, Asuman},
  journal={arXiv preprint arXiv:2002.07948},
  year={2020}
}

@article{jordan1994hierarchical,
  title={Hierarchical mixtures of experts and the EM algorithm},
  author={Jordan, Michael I and Jacobs, Robert A},
  journal={Neural computation},
  volume={6},
  number={2},
  pages={181--214},
  year={1994},
  publisher={MIT Press}
}

@inproceedings{huang2022learn,
  title={Learn from others and be yourself in heterogeneous federated learning},
  author={Huang, Wenke and Ye, Mang and Du, Bo},
  booktitle={Proceedings of the IEEE/CVF Conference on Computer Vision and Pattern Recognition},
  pages={10143--10153},
  year={2022}
}

@inproceedings{mendieta2022local,
  title={Local learning matters: Rethinking data heterogeneity in federated learning},
  author={Mendieta, Matias and Yang, Taojiannan and Wang, Pu and Lee, Minwoo and Ding, Zhengming and Chen, Chen},
  booktitle={Proceedings of the IEEE/CVF Conference on Computer Vision and Pattern Recognition},
  pages={8397--8406},
  year={2022}
}

@article{wang2024taming,
  title={Taming Cross-Domain Representation Variance in Federated Prototype Learning with Heterogeneous Data Domains},
  author={Wang, Lei and Bian, Jieming and Zhang, Letian and Chen, Chen and Xu, Jie},
  journal={arXiv preprint arXiv:2403.09048},
  year={2024}
}

@article{chung2024scaling,
  title={Scaling instruction-finetuned language models},
  author={Chung, Hyung Won and Hou, Le and Longpre, Shayne and Zoph, Barret and Tay, Yi and Fedus, William and Li, Yunxuan and Wang, Xuezhi and Dehghani, Mostafa and Brahma, Siddhartha and others},
  journal={Journal of Machine Learning Research},
  volume={25},
  number={70},
  pages={1--53},
  year={2024}
}

@article{tian2024hydralora,
  title={HydraLoRA: An Asymmetric LoRA Architecture for Efficient Fine-Tuning},
  author={Tian, Chunlin and Shi, Zhan and Guo, Zhijiang and Li, Li and Xu, Chengzhong},
  journal={arXiv preprint arXiv:2404.19245},
  year={2024}
}

@article{lermen2023lora,
  title={Lora fine-tuning efficiently undoes safety training in llama 2-chat 70b},
  author={Lermen, Simon and Rogers-Smith, Charlie and Ladish, Jeffrey},
  journal={arXiv preprint arXiv:2310.20624},
  year={2023}
}

@article{wang2023multilora,
  title={Multilora: Democratizing lora for better multi-task learning},
  author={Wang, Yiming and Lin, Yu and Zeng, Xiaodong and Zhang, Guannan},
  journal={arXiv preprint arXiv:2311.11501},
  year={2023}
}

@article{sheng2023s,
  title={S-lora: Serving thousands of concurrent lora adapters},
  author={Sheng, Ying and Cao, Shiyi and Li, Dacheng and Hooper, Coleman and Lee, Nicholas and Yang, Shuo and Chou, Christopher and Zhu, Banghua and Zheng, Lianmin and Keutzer, Kurt and others},
  journal={arXiv preprint arXiv:2311.03285},
  year={2023}
}

@article{liu2023moelora,
  title={Moelora: An moe-based parameter efficient fine-tuning method for multi-task medical applications},
  author={Liu, Qidong and Wu, Xian and Zhao, Xiangyu and Zhu, Yuanshao and Xu, Derong and Tian, Feng and Zheng, Yefeng},
  journal={arXiv preprint arXiv:2310.18339},
  year={2023}
}

@inproceedings{wang2024aggregation,
  title={An aggregation-free federated learning for tackling data heterogeneity},
  author={Wang, Yuan and Fu, Huazhu and Kanagavelu, Renuga and Wei, Qingsong and Liu, Yong and Goh, Rick Siow Mong},
  booktitle={Proceedings of the IEEE/CVF Conference on Computer Vision and Pattern Recognition},
  pages={26233--26242},
  year={2024}
}

@article{yang2024fedfed,
  title={FedFed: Feature distillation against data heterogeneity in federated learning},
  author={Yang, Zhiqin and Zhang, Yonggang and Zheng, Yu and Tian, Xinmei and Peng, Hao and Liu, Tongliang and Han, Bo},
  journal={Advances in Neural Information Processing Systems},
  volume={36},
  year={2024}
}

@article{bian2024accelerating,
  title={Accelerating hybrid federated learning convergence under partial participation},
  author={Bian, Jieming and Wang, Lei and Yang, Kun and Shen, Cong and Xu, Jie},
  journal={IEEE Transactions on Signal Processing},
  year={2024},
  publisher={IEEE}
}

@article{hadi2023survey,
  title={A survey on large language models: Applications, challenges, limitations, and practical usage},
  author={Hadi, Muhammad Usman and Qureshi, Rizwan and Shah, Abbas and Irfan, Muhammad and Zafar, Anas and Shaikh, Muhammad Bilal and Akhtar, Naveed and Wu, Jia and Mirjalili, Seyedali and others},
  journal={Authorea Preprints},
  year={2023},
  publisher={Authorea}
}

@article{bai2024federated,
  title={Federated fine-tuning of large language models under heterogeneous language tasks and client resources},
  author={Bai, Jiamu and Chen, Daoyuan and Qian, Bingchen and Yao, Liuyi and Li, Yaliang},
  journal={arXiv e-prints},
  pages={arXiv--2402},
  year={2024}
}

@inproceedings{cho2024heterogeneous,
  title={Heterogeneous lora for federated fine-tuning of on-device foundation models},
  author={Cho, Yae Jee and Liu, Luyang and Xu, Zheng and Fahrezi, Aldi and Joshi, Gauri},
  booktitle={Proceedings of the 2024 Conference on Empirical Methods in Natural Language Processing},
  pages={12903--12913},
  year={2024}
}

@inproceedings{kuang2024federatedscope,
  title={Federatedscope-llm: A comprehensive package for fine-tuning large language models in federated learning},
  author={Kuang, Weirui and Qian, Bingchen and Li, Zitao and Chen, Daoyuan and Gao, Dawei and Pan, Xuchen and Xie, Yuexiang and Li, Yaliang and Ding, Bolin and Zhou, Jingren},
  booktitle={Proceedings of the 30th ACM SIGKDD Conference on Knowledge Discovery and Data Mining},
  pages={5260--5271},
  year={2024}
}

@inproceedings{wu2024fedbiot,
  title={Fedbiot: Llm local fine-tuning in federated learning without full model},
  author={Wu, Feijie and Li, Zitao and Li, Yaliang and Ding, Bolin and Gao, Jing},
  booktitle={Proceedings of the 30th ACM SIGKDD Conference on Knowledge Discovery and Data Mining},
  pages={3345--3355},
  year={2024}
}

@article{sun2022conquering,
  title={Conquering the communication constraints to enable large pre-trained models in federated learning},
  author={Sun, Guangyu and Khalid, Umar and Mendieta, Matias and Yang, Taojiannan and Wang, Pu and Lee, Minwoo and Chen, Chen},
  journal={arXiv preprint arXiv:2210.01708},
  year={2022}
}

@article{yao2022benchmark,
  title={A benchmark for federated hetero-task learning},
  author={Yao, Liuyi and Gao, Dawei and Wang, Zhen and Xie, Yuexiang and Kuang, Weirui and Chen, Daoyuan and Wang, Haohui and Dong, Chenhe and Ding, Bolin and Li, Yaliang},
  journal={arXiv preprint arXiv:2206.03436},
  year={2022}
}

@article{le2023bloom,
  title={Bloom: A 176b-parameter open-access multilingual language model},
  author={Le Scao, Teven and Fan, Angela and Akiki, Christopher and Pavlick, Ellie and Ili{\'c}, Suzana and Hesslow, Daniel and Castagn{\'e}, Roman and Luccioni, Alexandra Sasha and Yvon, Fran{\c{c}}ois and Gall{\'e}, Matthias and others},
  year={2023}
}

@article{hao2025personalized,
  title={Personalized Federated Fine-tuning for Heterogeneous Data: An Automatic Rank Learning Approach via Two-Level LoRA},
  author={Hao, Jie and Wu, Yuman and Payani, Ali and Lee, Myungjin and Liu, Mingrui},
  journal={arXiv preprint arXiv:2503.03920},
  year={2025}
}

@article{bian2025survey,
  title={A survey on parameter-efficient fine-tuning for foundation models in federated learning},
  author={Bian, Jieming and Peng, Yuanzhe and Wang, Lei and Huang, Yin and Xu, Jie},
  journal={arXiv preprint arXiv:2504.21099},
  year={2025}
}

@article{zhang2025fedel,
  title={FedEL: Federated Elastic Learning for Heterogeneous Devices},
  author={Zhang, Letian and Chen, Bo and Bian, Jieming and Wang, Lei and Xu, Jie},
  journal={arXiv preprint arXiv:2509.16902},
  year={2025}
}

@inproceedings{10.1145/3746027.3755226,
author = {Liu, Junkang and Shang, Fanhua and Tian, Yuxuan and Liu, Hongying and Liu, Yuanyuan},
title = {Consistency of Local and Global Flatness for Federated Learning},
year = {2025},
isbn = {9798400720352},
publisher = {Association for Computing Machinery},
address = {New York, NY, USA},
url = {https://doi.org/10.1145/3746027.3755226},
doi = {10.1145/3746027.3755226},
booktitle = {Proceedings of the 33rd ACM International Conference on Multimedia},
pages = {3875–3883},
numpages = {9},
location = {Dublin, Ireland},
series = {MM '25}
}

@inproceedings{liu2024fedbcgd,
  title={Fedbcgd: Communication-efficient accelerated block coordinate gradient descent for federated learning},
  author={Liu, Junkang and Shang, Fanhua and Liu, Yuanyuan and Liu, Hongying and Li, Yuangang and Gong, YunXiang},
  booktitle={Proceedings of the 32nd ACM International Conference on Multimedia},
  pages={2955--2963},
  year={2024}
}

@inproceedings{liuimproving,
  title={Improving Generalization in Federated Learning with Highly Heterogeneous Data via Momentum-Based Stochastic Controlled Weight Averaging},
  author={Liu, Junkang and Liu, Yuanyuan and Shang, Fanhua and Liu, Hongying and Liu, Jin and Feng, Wei},
  booktitle={Forty-second International Conference on Machine Learning},
  year={2025}
}

@misc{liu2025fedadamwcommunicationefficientoptimizerconvergence,
      title={FedAdamW: A Communication-Efficient Optimizer with Convergence and Generalization Guarantees for Federated Large Models}, 
      author={Junkang Liu and Fanhua Shang and Kewen Zhu and Hongying Liu and Yuanyuan Liu and Jin Liu},
      year={2025},
      eprint={2510.27486},
      archivePrefix={arXiv},
      primaryClass={cs.LG},
      url={https://arxiv.org/abs/2510.27486}, 
}

@InProceedings{Bian_2025_ICCV,
    author    = {Bian, Jieming and Wang, Lei and Zhang, Letian and Xu, Jie},
    title     = {LoRA-FAIR: Federated LoRA Fine-Tuning with Aggregation and Initialization Refinement},
    booktitle = {Proceedings of the IEEE/CVF International Conference on Computer Vision (ICCV)},
    month     = {October},
    year      = {2025},
    pages     = {3737-3746}
}

\clearpage


\clearpage
\appendix

\section{Datasets}
\label{appendix: dataset}

In this paper, we follow the setup proposed in \cite{yang2024dual} to construct federated datasets for heterogeneous client settings using the Flan collection \cite{chung2024scaling}. The Flan benchmarks comprise over 12 NLP task categories, each containing multiple datasets. These tasks, derived from diverse contextual domains, inherently involve complex distribution shifts. We consider three federated dataset configurations. In the \textit{High Heterogeneous} setting, 8 clients are assigned entirely different NLP tasks, resulting in a total of 8 distinct task types. In the \textit{Low Heterogeneous} setting, 8 clients are distributed across only two tasks—Sentiment Analysis and Natural Language Inference—with all clients working on the same task drawing data from a common source, thereby minimizing heterogeneity. The \textit{Mild Heterogeneous} setting assigns 8 clients to datasets spanning 6 different NLP tasks. To better simulate realistic FL environments, where each client has limited access to large-scale data, we apply a downsampling strategy: each client is assigned 600 training samples and 300 testing samples. For our extended evaluations with 16 and 24 clients, we maintain the same total dataset size while distributing it across more clients. Specifically, in the 16-client configuration, each client receives 300 training samples and 150 testing samples, while in the 24-client setting, each client is allocated 200 training samples and 100 testing samples. The specific tasks and datasets used in each configuration are provided in Tables \ref{tab:fl_tasks_datasets_1}, \ref{tab:fl_tasks_datasets_2}, \ref{tab:fl_tasks_datasets_3}. Although some datasets appear in multiple configurations, sample selections are made independently and randomly.

The original Flan tasks exhibit highly diverse formats, making standardized preprocessing and evaluation across LLMs challenging. Even within a single dataset (e.g., Sentiment140), sample formats can vary significantly, as illustrated in \cref{tab:federated_dataset_examples_SA,tab:federated_dataset_examples_ST}. To ensure consistency, we adopt a unified prompt template based on \cite{yang2024dual}, shown in \cref{table-template}. \textcolor{black}{For evaluation, we follow FedDPA \cite{yang2024dual} and use ROUGE-1 as the metric across all tasks, providing a unified measure of model performance. This choice is particularly appropriate as Flan inherently structures NLP tasks in a generative format, where ROUGE-1 effectively captures lexical overlap between generated outputs and reference answers. Using ROUGE-1 enables direct comparability across diverse tasks and clients, which is essential for analyzing federated learning dynamics. While task-specific metrics might offer more nuanced evaluation for individual tasks, they would introduce incomparability issues that complicate assessment of the overall federated learning performance, making ROUGE-1 a methodologically sound choice.}

\begin{table*}[ht]
\centering
\begin{adjustbox}{width=0.8\linewidth}
\begin{tabular}{c|c|c}
\toprule[1.25pt]
\textbf{Task} & \textbf{Dataset} & \textbf{Description} \\ \midrule
Text Classification & ag\_news\_subset & News classification into categories. \\ 
Natural Language Inference & snli & Determine relationships between sentence pairs. \\ 
Reading Comprehension & openbookqa & Open-book question answering. \\ 
Paraphrase Detection & glue\_mrpc & Identify semantic equivalence in sentences. \\ 
Commonsense Reasoning & story\_cloze & Choose the correct story ending. \\ 
Struct to Text & common\_gen & Generate text from concepts. \\ 
Sentiment Analysis & sentiment140 & Sentiment analysis of tweets. \\ 
Coreference Resolution & definite\_pronoun\_resolution & Resolve ambiguous pronouns. \\ 
\bottomrule[1.25pt]
\end{tabular}
\end{adjustbox}
\caption{Tasks and Datasets for High Heterogeneous}
\label{tab:fl_tasks_datasets_1}
\end{table*}

\begin{table*}[ht]
\centering
\begin{adjustbox}{width=0.8\linewidth}
\begin{tabular}{c|c|c}
\toprule[1.25pt]
\textbf{Task} & \textbf{Dataset} & \textbf{Description} \\ \midrule
Struct to Text 1 & common\_gen & Generate text from concepts. \\ 
Struct to Text 2 & dart & Generate structured text descriptions. \\ 
Coreference Resolution & definite\_pronoun\_resolution & Resolve ambiguous pronouns. \\ 
Natural Language Inference 1 & snli & Determine relationships between sentence pairs. \\ 
Natural Language Inference 2 & qnli & Answer questions based on context. \\ 
Paraphrase Detection & glue\_mrpc & Identify semantic equivalence in sentences. \\ 
Text Classification & ag\_news\_subset & Classify news into categories. \\ 
Sentiment Analysis & sst2 & Analyze sentiment of text. \\ 
\bottomrule[1.25pt]
\end{tabular}
\end{adjustbox}
\caption{Tasks and Datasets for Mild Heterogeneous}
\label{tab:fl_tasks_datasets_2}
\end{table*}

\begin{table*}[ht]
\centering
\begin{adjustbox}{width=0.8\linewidth}
\begin{tabular}{c|c|c}
\toprule[1.25pt]
\textbf{Task} & \textbf{Dataset} & \textbf{Description} \\ \midrule
Sentiment Analysis 1 & sst2 & Analyze sentiment of short movie reviews. \\
Sentiment Analysis 2 & sentiment140 & Sentiment analysis of tweets. \\
Natural Language Inference 1 & qnli & Answer questions based on context. \\
Natural Language Inference 2 & snli & Determine relationships between sentence pairs. \\
\bottomrule[1.25pt]
\end{tabular}
\end{adjustbox}
\caption{Tasks and Datasets for Low Heterogeneous}
\label{tab:fl_tasks_datasets_3}
\end{table*}

\begin{table}[ht]
\centering
\begin{adjustbox}{width=\linewidth}
\begin{tabular}{p{11.5cm}} 
\toprule
\textbf{Data Examples} \\
\midrule
\textbf{Instruction:} @celinieee The crafts thingy is already done. Describe the sentiment embodied by this tweet. OPTIONS: - negative - positive. \\
\textbf{Output:} negative \\
\textbf{Dataset:} sentiment140 \\
\midrule
\textbf{Instruction:} Write a positive tweet. \\
\textbf{Output:} @MandinaM Very well indeed thank you.. Yourself? Just about to start work.. 9 hours to go.. \\
\textbf{Dataset:} sentiment140 \\
\bottomrule
\end{tabular}
\end{adjustbox}
\caption{Examples of Data in Sentiment Analysis Task.}
\label{tab:federated_dataset_examples_SA}
\end{table}

\begin{table}[ht]
\centering
\begin{adjustbox}{width=\linewidth}
\begin{tabular}{p{11.5cm}} 
\toprule
\textbf{Data Examples} \\
\midrule
\textbf{Instruction:} What are the most important words in the following sentence: pig in a pen on a pig farm \\
\textbf{Output:} farm, pen, pig \\
\textbf{Dataset:} common\_gen \\
\midrule
\textbf{Instruction:} Generate a sentence about this data: Wildwood eatType coffee shop; Wildwood priceRange moderate; Wildwood customer rating 1 out of 5; Wildwood near Ranch \\
\textbf{Output:} The Wildwood coffee shop is an average-priced location that serves a light menu of food and alcohol. Located near the Ranch, customers give this location a 1 out of 5 stars. \\
\textbf{Dataset:} dart \\
\bottomrule
\end{tabular}
\end{adjustbox}
\caption{Examples of Data in Struct to Text Task.}
\label{tab:federated_dataset_examples_ST}
\end{table}

\begin{table}[ht]
    \centering
    \begin{tabular}{ll}
        \toprule[1.25pt]
        & \multirow{1}{*}{\textbf{Template}} \\ \hline
        Prompt Input & \makecell[l]{\textbf{Instruction}: \{instruction\} \\ \textbf{Response}:} \\
        \bottomrule[1.25pt]
    \end{tabular}
    \caption{Prompt Template.}
    \label{table-template}
\end{table}

\section{Training Implementations}
\label{appendix: training}

We implement all methods using PyTorch and conduct all experiments on six NVIDIA A100 GPUs. For the federated methods, we set the local epochs to 5 and the global rounds to 20. For the \textit{Local Only} baseline, the total number of training rounds is set to 50. The learning rate is explored in the range \(\{3 \times 10^{-2}, 1 \times 10^{-3}, 3 \times 10^{-3}, 1 \times 10^{-4}, 3 \times 10^{-4}, 1 \times 10^{-5}\}\), and the AdamW optimizer is utilized, following \cite{yang2024dual}. 

\textcolor{black}{LoRA modules are inserted into the \(q\) and \(v\) components of each self-attention layer. For LLaMA2 experiments, we set the finetuning batch size to 8 and the inference batch size to 2. For Bloom experiments, both finetuning and inference batch sizes are set to 32. The rank of the trainable LoRA for each method is primarily set to 8, with additional experiments using ranks 4 and 6 conducted for sensitivity analysis. We configure other LoRA hyperparameters with \texttt{lora\_alpha} set to 32 and \texttt{lora\_dropout} set to 0.05. Due to computational constraints, the pre-trained LLaMA2 model is loaded using 8-bit quantization to reduce memory requirements while maintaining performance.}

\section{Compared Methods}
\label{appendix: methods}
A detailed overview of the baseline methods used for comparison in our experiments is provided below.

\textbf{FedIT \cite{zhang2024towards}.}  
FedIT integrates LoRA with the FedAvg framework, making it one of the earliest approaches to combine these techniques. However, under heterogeneous client data settings, FedIT suffers from harmful cross-client interference during the aggregation step, as highlighted in the motivational experiments. This interference often reduces its efficiency and, in some cases, results in worse performance than the \textit{Local Only} method.

\textbf{FFA-LoRA \cite{sun2024improving}.}  
FFA-LoRA aims to reduce communication costs by fixing the randomly initialized non-zero \(\mathbf{A}\) matrices and fine-tuning only the zero-initialized \(\mathbf{B}\) matrices. While this approach lowers communication overhead, it impairs the learning capacity of LoRA, as fixed \(\mathbf{A}\) matrices limit the model’s ability to adapt to diverse client data, resulting in suboptimal performance.

\textbf{FedSA \cite{guo2024selective}.}  
FedSA introduces an asymmetric approach to LoRA matrix updates in federated learning, where only the \(\mathbf{A}\) matrices are shared with the server for aggregation, while the \(\mathbf{B}\) matrices remain local to capture client-specific knowledge. However, even the client-specific \(\mathbf{B}\) matrices may contain valuable knowledge that could benefit other clients. The failure to leverage this knowledge effectively results in reduced overall performance.

\textbf{PF2LoRA \cite{hao2025personalized}.}  
PF2LoRA adopts a two-level low-rank adaptation framework, where a common adapter shared across all clients is combined with client-specific lightweight adapters for personalization. This bilevel optimization framework aims to balance shared knowledge and personalization. However, the aggregation of the shared adapter during the server aggregation step can reintroduce harmful interference, limiting the method's effectiveness in highly heterogeneous settings.

\textbf{FDLoRA \cite{qi2024fdlora}.}  
FDLoRA follows a two-step process where personalized LoRA modules are first trained locally and then used to initialize federated training to obtain a global LoRA module. An adaptive fusion mechanism combines the two LoRA modules during fine-tuning. While this method improves personalization, it requires additional server-side datasets for effective global LoRA training. Moreover, the aggregation step during global LoRA training remains prone to interference, reducing its robustness in handling client heterogeneity.

\textbf{FedDPA \cite{yang2024dual}.}  
FedDPA is the most closely related to our proposed method, as it also targets settings where clients possess heterogeneous tasks (not just label non-IID data). FedDPA employs both global and local LoRA modules, and enhances personalization by introducing a local fine-tuning step excluded from aggregation. However, our proposed FedALT departs from FedDPA in several critical ways: \textbf{First}, in both variants of FedDPA, the global LoRA module is trained independently of the local LoRA modules. During global training, the local modules are excluded, leading to persistent interference—similar to FedIT—and the local LoRA modules are never shared across clients. In contrast, FedALT introduces the Rest-of-World (RoW) LoRA, which captures knowledge from other clients via aggregation, allowing more effective information sharing. \textbf{Second}, FedALT employs a dynamic mixer that enables each client to adaptively balance the contributions of its Individual and RoW LoRA modules on a per-input basis, enabling finer-grained personalization. \textbf{Third}, FedDPA separates local and global LoRA training into distinct phases, introducing additional hyperparameters to tune the duration of each phase. Short local training reduces personalization; overly long training approximates local-only fine-tuning, diminishing the benefits of FL. In contrast, FedALT removes the need for phase separation and additional hyperparameters, simplifying training and improving robustness.

\section{Discussion: Why a Single RoW LoRA?}  
FedALT prioritizes local training while selectively incorporating useful information from other clients via a dynamic weighting mechanism. Ideally, each client could tailor this incorporation by maintaining separate LoRA components for every other client. In this approach, client \( k \) would receive the local LoRA components \( \mathbf{A}^L_j / \mathbf{B}^L_j \) from all other clients \( j \neq k \), freeze them during local training, and learn separate weights for each \( j \): 
\[
y = \mathbf{W}_0 x + \alpha_{k, k}(x) \mathbf{B}^L_k \mathbf{A}^L_k x + \sum_{j\neq k} \alpha_{k,j}(x) \mathbf{B}^L_j \mathbf{A}^L_j x
\]
where \( \sum_{j} \alpha_{k,j}(x) = 1 \). This would allow client \( k \) to capture the distinct influence of each other client's LoRA. However, this approach presents significant drawbacks. First, each client would need to store and communicate a LoRA component for every other client, leading to memory and bandwidth overhead that scales linearly with the number of clients \( K \), making it impractical for large-scale settings. Second, training the mixer becomes increasingly complex with more LoRA components, requiring additional parameters and longer convergence times.  To address these challenges, we adopt a single RoW LoRA, which aggregates knowledge from other clients into a single component. While this reduces computational and communication costs, it may limit personalization. Future work could explore clustering-based approaches to maintain a small number of RoW LoRA components, capturing finer-grained client relationships while remaining scalable.

\section{The Algorithm of FedALT}
We summarize the pseudocode of FedALT in Algorithms 1.



\begin{algorithm}[h]
\caption{FedALT}
\begin{algorithmic}[1]
\REQUIRE Pretrained model $W_0$, number of clients $K$, communication rounds $T$
\FOR{each client $k$}
    \STATE Initialize LoRA $(A_k^L, B_k^L)$ and mixer $G_k$ randomly
    \STATE Initialize RoW LoRA $(A_k^R, B_k^R) \gets 0$
\ENDFOR
\FOR{$t = 1$ to $T$}
    \STATE \textbf{Client-side training}:
    \FOR{each client $k$ \textbf{in parallel}}
        \STATE Receive RoW LoRA $(A_k^R, B_k^R)$ from server
        \FOR{each local epoch}
            \STATE Update $(A_k^L, B_k^L)$ and $G_k$ using gradient descent
            \STATE Compute output:
            \STATE $y = W_0 x + \alpha_k(x) B_k^L A_k^L x + (1 - \alpha_k(x)) B_k^R A_k^R x$
            \STATE where $\alpha_k(x) = \text{softmax}(G_k x)$
        \ENDFOR
        \STATE Upload $(A_k^L, B_k^L)$ to server
    \ENDFOR
    \STATE \textbf{Server-side aggregation}:
    \FOR{each client $k$}
        \STATE $A_k^R \gets \frac{1}{K-1} \sum_{m \ne k} A_m^L$
        \STATE $B_k^R \gets \frac{1}{K-1} \sum_{m \ne k} B_m^L$
        \STATE Send $(A_k^R, B_k^R)$ to client $k$
    \ENDFOR
\ENDFOR
\end{algorithmic}
\end{algorithm}

\section{Additional Experiment Results}
\label{appendix: experiment}

\subsection{Training Workload}
In \cref{tab:trainable_param}, we compare the training workload of our proposed FedALT with existing baselines using the Bloom-560M model. We report both the percentage of trainable parameters, the average local epoch training time and the inference time. FedIT serves as a representative baseline, as most federated methods share similar architecture in terms of parameter count and inference cost. Among all methods, FFA-LoRA has the fewest trainable parameters, as it updates only the LoRA \( \mathbf{B} \) matrices; however, this results in significantly lower performance. In contrast, FedALT delivers substantially better results while maintaining comparable training/inference time and parameter overhead, demonstrating that it achieves efficiency without sacrificing effectiveness.

To further ensure that FedALT’s performance gains are not simply due to having more parameters, we conduct two controlled comparisons, as shown in \cref{tab:trainable_vs_perf}. In \textit{RoW-Update (rank=8)}, both Individual and RoW LoRA modules (each with rank 8) are trained jointly with a dynamic mixer. In \textit{Fixed Weight (rank=16)}, only the Individual LoRA (rank 16) is trained, while the frozen RoW LoRA is combined with a fixed mixing weight \(\alpha = 0.5\). Although both configurations involve more trainable parameters than FedALT, they yield lower performance. These comparisons validate two core design choices in FedALT: the effectiveness of decoupled training and the benefit of dynamic mixing for personalized adaptation.

\begin{table}[h]
\centering
\begin{adjustbox}{width=1\linewidth}
\begin{tabular}{c|c|c|c}
\toprule
\textbf{Methods} & \textbf{Trainable Parameters} & \textbf{Avg Local Epoch Time} &  \textbf{Inference Time} \\
\midrule
FedIT      & 0.0622\% & 6.66 s & 8.89 s\\
FFA-LoRA   & 0.0311\% & 6.56 s & 8.77 s\\
\textbf{FedALT} & 0.0661\% & 7.05 s  & 8.94 s\\
\bottomrule
\end{tabular}
\end{adjustbox}
\caption{\textbf{Comparison of training workload.} FedALT uses a comparable number of LoRA parameters to other baselines.}
\label{tab:trainable_param}
\end{table}

\begin{table*}[h]
\centering
\begin{adjustbox}{width=0.5\linewidth}
\begin{tabular}{l|c|c}
\toprule
\textbf{Methods} & \textbf{Trainable Parameters} & \textbf{Avg Performance} \\
\midrule
RoW-Update (rank=8) & 0.1283\% & 65.75 \\
Fixed Weight (rank=16) & 0.1244\% & 64.59 \\
\textbf{FedALT} & \textbf{0.0661\%} & \textbf{67.55} \\
\bottomrule
\end{tabular}
\end{adjustbox}
\caption{\textbf{Comparison of trainable parameters and performance.} FedALT achieves the highest performance while using fewer parameters.}
\label{tab:trainable_vs_perf}
\end{table*}

\subsection{Effect of Data Heterogeneity}
To further evaluate the robustness of our method under varying degrees of client heterogeneity, we design experiments across three settings. Our main experimental configuration, where each client is assigned a unique task, is denoted as the \textit{High Heterogeneous} setting. We also include two alternative configurations. In the \textit{Low Heterogeneous} setting, 8 clients are divided between only two tasks—Sentiment Analysis and Natural Language Inference—with clients sharing a task drawing data from the same source, resulting in minimal heterogeneity. The \textit{Mild Heterogeneous} setting involves 8 clients assigned datasets spanning a total of 6 distinct NLP tasks. As shown in \cref{tab:performance_mild_llama,tab:performance_mild_bloom,tab:performance_least_llama_8clients,tab:performance_least_bloom_8clients}., FedALT consistently outperforms all baseline methods across all levels of heterogeneity, using both Bloom and LLaMA2 as pre-trained models. These results highlight the robustness and adaptability of FedALT in handling diverse and heterogeneous data distributions.  

\begin{table*}[h!]
\centering
\resizebox{0.85\textwidth}{!}{%
\begin{tabular}{l|cccccccc|c}
\toprule
\multirow{2}{*}{\textbf{Methods}} & \multicolumn{8}{c|}{\textbf{LLaMA2-7B}} & \multirow{2}{*}{\textbf{Average}} \\
\cline{2-9}
& Struct 1 & Struct 2 & Coref. & NLI 1 & NLI 2 & Para. & Text Cls. & Sentiment & \\
\midrule
Local Only  & \textbf{52.98} & 57.35 & 74.08 & 66.10 & 69.02 & 78.67 & 64.97 & 70.16 & 66.67 \\
FedIT       & 48.30 & 52.81 & 76.51 & 65.61 & 69.43 & 71.56 & 60.08 & 71.90 & 64.53 \\
FFA-LoRA    & 50.64 & 51.70 & 68.09 & 60.99 & 64.59 & 70.00 & 61.47 & 69.68 & 62.15 \\
FedSA       & 52.22 & 56.59 & 78.11 & 71.41 & 69.13 & 80.00 & 65.15 & 66.92 & 67.44 \\
FedDPA      & 51.99 & 55.62 & 78.60 & 66.82 & 68.49 & 78.67 & 66.09 & 73.60 & 67.49 \\
PF2LoRA       & 50.70 & 60.19 & 77.80 & 68.07 & 69.21 & 72.67 & \textbf{71.10} & 71.07 & 67.60 \\
FDLoRA       & 52.71 & 58.37 & 74.74 & 69.62 & 66.27 & 71.00 & 65.99 & \textbf{73.80} & 66.56 \\
\rowcolor{gray!40} \textbf{FedALT}      & 51.73 & \textbf{60.72} & \textbf{81.91} & \textbf{74.51} & \textbf{71.12} & \textbf{80.33} & 70.28 & 73.22 & \textbf{70.48} \\
\bottomrule
\end{tabular}%
}
\caption{\textbf{Performance on mild heterogeneous (LLaMA2-7B).} Evaluation across tasks in Table~\ref{tab:fl_tasks_datasets_2}.}
\label{tab:performance_mild_llama}
\end{table*}

\begin{table*}[h!]
\centering
\resizebox{0.85\textwidth}{!}{%
\begin{tabular}{l|cccccccc|c}
\toprule
\multirow{2}{*}{\textbf{Methods}} & \multicolumn{8}{c|}{\textbf{Bloom-560M}} & \multirow{2}{*}{\textbf{Average}} \\
\cline{2-9}
& Struct 1 & Struct 2 & Coref. & NLI 1 & NLI 2 & Para. & Text Cls. & Sentiment & \\
\midrule
Local       & 40.47 & 48.10 & 34.48 & 42.33 & 52.58 & 70.26 & 47.65 & 65.82 & 50.21 \\
FedIT       & 35.38 & 41.20 & 38.86 & 39.23 & 49.16 & \textbf{71.67} & 51.75 & 65.00 & 49.03 \\
FFA-LoRA    & 36.16 & 39.40 & 36.57 & 36.22 & 47.19 & 68.67 & 50.45 & 63.89 & 47.32 \\
FedSA       & 35.98 & 44.71 & 38.68 & 47.68 & 55.92 & 70.54 & 58.02 & 68.84 & 52.55 \\
FedDPA      & 38.56 & 42.54 & 39.23 & 46.69 & 49.21 & 69.00 & 62.35 & \textbf{68.92} & 52.06 \\
PF2LoRA     & 39.65 & 42.29 & \textbf{40.36} & 45.85 & 53.07 & 70.40 & 57.43 & 66.67 & 51.97 \\
FDLoRA      & 40.66 & 44.62 & 38.91 & 43.54 & 51.17 & 70.33 & 54.13 & 68.00 & 51.42 \\
\rowcolor{gray!25} \textbf{FedALT} & \textbf{41.10} & \textbf{48.64} & 39.80 & \textbf{48.14} & \textbf{56.63} & 71.27 & \textbf{64.84} & 68.67 & \textbf{54.89} \\
\bottomrule
\end{tabular}%
}
\caption{\textbf{Performance on mild heterogeneous (Bloom-560M).} Evaluation across tasks in Table~\ref{tab:fl_tasks_datasets_2}.}
\label{tab:performance_mild_bloom}
\end{table*}

\begin{table*}[h]
\centering
\resizebox{0.85\textwidth}{!}{%
\begin{tabular}{l|cccccccc|c}
\toprule
\textbf{Methods} & SA1 & SA2 & SA3 & SA4 & NLI1 & NLI2 & NLI3 & NLI4   & \textbf{Average} \\
\midrule
Local & 65.95 & 63.95 & 63.57 & 66.34 & 45.95 & 47.45 & 66.93 & 64.60 & 60.59 \\
FedIT & 67.25 & 66.82 & 65.85 & 64.79 & 46.53 & 49.48 & 68.53 & 68.65 & 62.24 \\
FFA-LoRA & 66.42 & 61.84 & 62.96 & 65.14 & 46.83 & 49.23 & 67.93 & 68.16 & 61.07 \\
FedSA & 71.85 & 64.31 & 65.45 & 69.37 & 46.31 & 48.24 & 67.03 & \textbf{70.38} & 62.87 \\
FedDPA & 67.39 & 64.74 & 65.58 & 63.98 & 46.89 & 48.63 & 69.77 & 69.47 & 62.06 \\
PF2LoRA & 71.73 & 67.83 & 65.45  & 70.18 & 47.19 & 47.68 & 69.41 & 69.03 & 63.56 \\
FDLoRA & 68.54 & 65.46 & \textbf{65.93} & 63.36 & 47.09 & 49.19 & \textbf{70.69} & 69.07 & 62.42 \\
\rowcolor{gray!25} \textbf{FedALT} & \textbf{73.05} & \textbf{67.86} & 65.91 & \textbf{72.82} & \textbf{47.29} & \textbf{50.01} & 70.42 & 69.88 & \textbf{64.66} \\
\bottomrule
\end{tabular}%
}
\caption{\textbf{Performance on low heterogeneous (LLaMA2-7B).} Evaluation across 8 clients split by task-dataset pairs. SA1–SA2 use \texttt{sst2}, SA3–SA4 use \texttt{sentiment140}, NLI1–NLI2 use \texttt{qnli}, and NLI3–NLI4 use \texttt{snli}.}
\label{tab:performance_least_llama_8clients}
\end{table*}

\begin{table*}[h!]
\centering
\resizebox{0.85\textwidth}{!}{%
\begin{tabular}{l|cccccccc|c}
\toprule
\textbf{Methods} & SA1 & SA2 & SA3 & SA4 & NLI1 & NLI2 & NLI3 & NLI4 & \textbf{Average} \\
\midrule
Local & 44.33 & 40.45 & 46.03 & 51.04 & 40.64 & 42.58 & 62.13 & 62.03 & 48.65 \\
FedIT       & 46.89 & 45.16 & 50.20 & 54.69 & 42.49 & \textbf{46.26} & 64.19 & 67.29 & 52.15 \\
FFA-LoRA    & 52.18 & 36.04 & 52.47 & 51.91  & 42.78 & 45.29 & 64.50 & 65.45 & 51.33 \\
FedSA       & 53.55 & 47.16 & 51.63 & 50.47 & \textbf{44.96} & 45.92 & 66.36 & 64.18 & 53.03 \\
FedDPA      & 52.61 & 47.27 & 52.03 & 53.14 & 43.26 & 45.56 & 65.68 & 65.38 & 53.12 \\
PF2LoRA     & 46.88 & 43.00 & \textbf{55.24} & 55.30 & 39.42 & 43.39 & 64.14 & 61.39 & 51.10 \\
FDLoRA      & 50.66 & 46.72 & 49.38 & 54.07 & 42.35 & 45.82 & 64.63 & 66.45 & 52.51 \\
\rowcolor{gray!40} \textbf{FedALT}      & \textbf{54.06} & \textbf{49.18} & 53.34 & \textbf{55.78}  & 42.94 & 45.51 & \textbf{67.57} & \textbf{67.49} & \textbf{54.48} \\
\bottomrule
\end{tabular}%
}
\caption{\textbf{Performance on low heterogeneous (Bloom-560M).} Each dataset is split across two clients. See Table~\ref{tab:fl_tasks_datasets_3} for dataset mapping.}
\label{tab:performance_least_bloom_8clients}
\end{table*}

\section{Convergence Analysis}
\label{appendix: convergence}
\textcolor{black}{Since our mixer depends on each sample input, existing convergence analysis frameworks cannot be directly applied. In this section, we provide an initial convergence analysis of FedALT under appropriate assumptions. We begin by defining the key parameters and establishing necessary conditions. While this analysis offers theoretical validation for our approach, we acknowledge that it represents a preliminary investigation into the convergence properties of FedALT. More comprehensive theoretical frameworks with potentially stronger guarantees and less restrictive assumptions remain an important direction for future research. Nevertheless, our initial analysis provides valuable insights into the convergence behavior and offers a theoretical foundation that complements our empirical findings.}

\subsection{Notation and Setup}

Let the trainable parameters for client $k$ at round $t$ be denoted by
\begin{equation}
Z_k^t = \{\mathbf{A}_k^L, \mathbf{B}_k^L, \mathbf{G}_k\}
\end{equation}

where $\mathbf{A}_k^L$ and $\mathbf{B}_k^L$ represent the Individual LoRA modules and $\mathbf{G}_k$ denotes the adaptive mixer.

The Rest-of-World (RoW) LoRA for client $k$ at round $t$ is computed as:
\begin{equation}
\mathbf{A}_k^R = \frac{1}{K-1}\sum_{m\neq k}\mathbf{A}_m^L, \quad \mathbf{B}_k^R = \frac{1}{K-1}\sum_{m\neq k}\mathbf{B}_m^L
\end{equation}

We denote the loss function for client $k$ as $L_k(Z_k|\mathbf{A}_k^R, \mathbf{B}_k^R)$, which measures the performance on client $k$'s local dataset given its trainable parameters and the fixed RoW LoRA components.

\subsection{Assumptions}

\begin{assumption}[Smoothness]
For each client $k$, the loss function $L_k$ is $\mu$-smooth with respect to $Z_k$, i.e., for any $Z_k^1, Z_k^2$:
\begin{equation}
\|\nabla L_k(Z_k^1) - \nabla L_k(Z_k^2)\| \leq \mu \|Z_k^1 - Z_k^2\|
\end{equation}
\end{assumption}

\begin{assumption}[Strong Convexity]
For each client $k$, the loss function $L_k$ is $\lambda$-strongly convex with respect to $Z_k$, i.e., for any $Z_k^1, Z_k^2$:
\begin{equation}
L_k(Z_k^2) \geq L_k(Z_k^1) + \langle \nabla L_k(Z_k^1), Z_k^2 - Z_k^1 \rangle + \frac{\lambda}{2}\|Z_k^2 - Z_k^1\|^2
\end{equation}
\end{assumption}

\begin{assumption}[Bounded RoW Sensitivity]
The optimal parameters for client $k$ are sensitive to changes in the RoW parameters with a bounded Lipschitz constant. Specifically, for any two sets of RoW parameters $(\mathbf{A}_k^{R,1}, \mathbf{B}_k^{R,1})$ and $(\mathbf{A}_k^{R,2}, \mathbf{B}_k^{R,2})$:
\begin{align}
& \|Z_k^*(\mathbf{A}_k^{R,1}, \mathbf{B}_k^{R,1}) - Z_k^*(\mathbf{A}_k^{R,2}, \mathbf{B}_k^{R,2})\| \nonumber \\
& \leq  \beta \|(\mathbf{A}_k^{R,1}, \mathbf{B}_k^{R,1}) - (\mathbf{A}_k^{R,2}, \mathbf{B}_k^{R,2})\|
\end{align}
where $Z_k^*(\mathbf{A}_k^R, \mathbf{B}_k^R) = \arg\min_{Z_k} L_k(Z_k|\mathbf{A}_k^R, \mathbf{B}_k^R)$ represents the optimal parameters for client $k$ given fixed RoW parameters.
\end{assumption}

The first two assumptions are standard and widely used in FL proof. The third assumption holds because the RoW LoRA component remains frozen during local training, providing a stable reference point that limits how much client parameters can drift between rounds based on changes in other clients' parameters.

\subsection{Convergence Theorem}

\begin{theorem}[Convergence of FedALT]
With the assumptions, we can derive:
\[
\| Z_k^{t+1} - Z_k^t \| \leq \frac{2\epsilon}{1 - \beta} + \beta^t \max_k \| Z_k^1 - Z_k^0 \|
\]

If $\beta < 1$ and the local training at each round converges to a neighborhood of the optimal solution, then the sequence of trainable parameters $\{Z_k^t\}$ generated by FedALT converges to a stable point for each client $k$.
\end{theorem}

\begin{proof}
We analyze the convergence in two steps:
\begin{enumerate}
    \item Local convergence within each round
    \item Global convergence across rounds
\end{enumerate}

\textbf{Step 1: Local convergence within each round}

At round $t$, client $k$ optimizes its parameters $Z_k$ with fixed RoW parameters $(\mathbf{A}_k^{R,t}, \mathbf{B}_k^{R,t})$. Let $Z_k^{t,s}$ represent the parameters after $s$ steps of local training. We now analyze the convergence of this local optimization process using Assumptions 1 and 2.

The gradient descent update rule at step $s$ is given by:
\begin{equation}
Z_k^{t,s+1} = Z_k^{t,s} - \eta \nabla L_k(Z_k^{t,s}|\mathbf{A}_k^{R,t}, \mathbf{B}_k^{R,t})
\end{equation}

From Assumption 1 (smoothness), we can have:
\begin{align}
&L_k(Z_k^{t,s+1}|\mathbf{A}_k^{R,t}, \mathbf{B}_k^{R,t}) \nonumber\\
&\leq L_k(Z_k^{t,s}|\mathbf{A}_k^{R,t}, \mathbf{B}_k^{R,t}) \nonumber\\
&\quad + \langle \nabla L_k(Z_k^{t,s}|\mathbf{A}_k^{R,t}, \mathbf{B}_k^{R,t}), Z_k^{t,s+1} - Z_k^{t,s} \rangle \nonumber\\
&\quad + \frac{\mu}{2}\|Z_k^{t,s+1} - Z_k^{t,s}\|^2
\end{align}

Substituting the update rule:
\begin{align}
& L_k(Z_k^{t,s+1}|\mathbf{A}_k^{R,t}, \mathbf{B}_k^{R,t}) \nonumber\\
&\leq L_k(Z_k^{t,s}|\mathbf{A}_k^{R,t}, \mathbf{B}_k^{R,t}) \nonumber\\
&\quad + \langle \nabla L_k(Z_k^{t,s}|\mathbf{A}_k^{R,t}, \mathbf{B}_k^{R,t}), -\eta \nabla L_k(Z_k^{t,s}|\mathbf{A}_k^{R,t}, \mathbf{B}_k^{R,t}) \rangle \nonumber\\
&\quad + \frac{\mu}{2}\|-\eta \nabla L_k(Z_k^{t,s}|\mathbf{A}_k^{R,t}, \mathbf{B}_k^{R,t})\|^2 \nonumber\\
&= L_k(Z_k^{t,s}|\mathbf{A}_k^{R,t}, \mathbf{B}_k^{R,t}) - \eta \|\nabla L_k(Z_k^{t,s}|\mathbf{A}_k^{R,t}, \mathbf{B}_k^{R,t})\|^2 \nonumber\\
&\quad + \frac{\mu\eta^2}{2}\|\nabla L_k(Z_k^{t,s}|\mathbf{A}_k^{R,t}, \mathbf{B}_k^{R,t})\|^2 \nonumber\\
&= L_k(Z_k^{t,s}|\mathbf{A}_k^{R,t}, \mathbf{B}_k^{R,t})\nonumber\\
&\quad - \eta(1 - \frac{\mu\eta}{2})\|\nabla L_k(Z_k^{t,s}|\mathbf{A}_k^{R,t}, \mathbf{B}_k^{R,t})\|^2
\end{align}

Let $Z_k^* = Z_k^*(\mathbf{A}_k^{R,t}, \mathbf{B}_k^{R,t})$ denote the optimal parameters that minimize the loss function for client $k$ given the fixed RoW parameters. 

From Assumption 2 (strong convexity), we have:
\begin{align}
&L_k(Z_k^*|\mathbf{A}_k^{R,t}, \mathbf{B}_k^{R,t})  \nonumber\\
&\geq L_k(Z_k^{t,s}|\mathbf{A}_k^{R,t}, \mathbf{B}_k^{R,t}) 
\nonumber\\
&\quad + \langle \nabla L_k(Z_k^{t,s}|\mathbf{A}_k^{R,t}, \mathbf{B}_k^{R,t}), Z_k^* - Z_k^{t,s} \rangle \nonumber\\
&\quad + \frac{\lambda}{2}\|Z_k^* - Z_k^{t,s}\|^2
\end{align}

Rearranging:
\begin{align}
& \langle \nabla L_k(Z_k^{t,s}|\mathbf{A}_k^{R,t}, \mathbf{B}_k^{R,t}), Z_k^{t,s} - Z_k^* \rangle \nonumber\\
& \geq  L_k(Z_k^{t,s}|\mathbf{A}_k^{R,t}, \mathbf{B}_k^{R,t})  - L_k(Z_k^*|\mathbf{A}_k^{R,t}, \mathbf{B}_k^{R,t}) \nonumber\\
& \quad+ \frac{\lambda}{2}\|Z_k^* - Z_k^{t,s}\|^2
\end{align}

By Cauchy-Schwarz inequality:
\begin{align}
& \|\nabla L_k(Z_k^{t,s}|\mathbf{A}_k^{R,t}, \mathbf{B}_k^{R,t})\| \cdot \|Z_k^{t,s} - Z_k^*\| \nonumber \\
& \geq  \langle \nabla L_k(Z_k^{t,s}|\mathbf{A}_k^{R,t}, \mathbf{B}_k^{R,t}), Z_k^{t,s} - Z_k^* \rangle
\end{align}

Combining with the previous inequality:
\begin{align}
&\|\nabla L_k(Z_k^{t,s}|\mathbf{A}_k^{R,t}, \mathbf{B}_k^{R,t})\| \cdot \|Z_k^{t,s} - Z_k^*\| \nonumber \\
& \geq  L_k(Z_k^{t,s}|\mathbf{A}_k^{R,t}, \mathbf{B}_k^{R,t}) - L_k(Z_k^*|\mathbf{A}_k^{R,t}, \mathbf{B}_k^{R,t}) \nonumber \\
&\quad+ \frac{\lambda}{2}\|Z_k^* - Z_k^{t,s}\|^2
\end{align}

From the properties of strongly convex functions, we can derive:
\begin{align}
& \|\nabla L_k(Z_k^{t,s}|\mathbf{A}_k^{R,t}, \mathbf{B}_k^{R,t})\|^2 \nonumber\\
&\geq 2\lambda(L_k(Z_k^{t,s}|\mathbf{A}_k^{R,t}, \mathbf{B}_k^{R,t}) - L_k(Z_k^*|\mathbf{A}_k^{R,t}, \mathbf{B}_k^{R,t}))
\end{align}

Substituting this into our earlier expression for the loss after one update:
\begin{align}
&L_k(Z_k^{t,s+1}|\mathbf{A}_k^{R,t}, \mathbf{B}_k^{R,t}) - L_k(Z_k^*|\mathbf{A}_k^{R,t}, \mathbf{B}_k^{R,t}) \nonumber\\
&\leq L_k(Z_k^{t,s}|\mathbf{A}_k^{R,t}, \mathbf{B}_k^{R,t}) - L_k(Z_k^*|\mathbf{A}_k^{R,t}, \mathbf{B}_k^{R,t}) \nonumber \\
&\quad- \eta(1 - \frac{\mu\eta}{2})\|\nabla L_k(Z_k^{t,s}|\mathbf{A}_k^{R,t}, \mathbf{B}_k^{R,t})\|^2\\
&\leq L_k(Z_k^{t,s}|\mathbf{A}_k^{R,t}, \mathbf{B}_k^{R,t}) - L_k(Z_k^*|\mathbf{A}_k^{R,t}, \mathbf{B}_k^{R,t}) \nonumber\\
&\quad - 2\eta(1 - \frac{\mu\eta}{2})\lambda(L_k(Z_k^{t,s}|\mathbf{A}_k^{R,t}, \mathbf{B}_k^{R,t}) \nonumber\\
&\quad - L_k(Z_k^*|\mathbf{A}_k^{R,t}, \mathbf{B}_k^{R,t}))\\
& = (1 - 2\eta\lambda(1 - \frac{\mu\eta}{2}))(L_k(Z_k^{t,s}|\mathbf{A}_k^{R,t}, \mathbf{B}_k^{R,t}) \nonumber\\
&\quad - L_k(Z_k^*|\mathbf{A}_k^{R,t}, \mathbf{B}_k^{R,t})),
\end{align}
where the learning rate $\eta <\frac{2}{\mu}$.

From Assumption 2 (strong convexity), we have:
\begin{align}
&\frac{\lambda}{2}\|Z_k^{t,s} - Z_k^*\|^2 \nonumber\\
&\leq L_k(Z_k^{t,s}|\mathbf{A}_k^{R,t}, \mathbf{B}_k^{R,t}) - L_k(Z_k^*|\mathbf{A}_k^{R,t}, \mathbf{B}_k^{R,t})
\end{align}

And from Assumption 1 (smoothness):
\begin{align}
&L_k(Z_k^{t,s}|\mathbf{A}_k^{R,t}, \mathbf{B}_k^{R,t}) - L_k(Z_k^*|\mathbf{A}_k^{R,t}, \mathbf{B}_k^{R,t}) \nonumber\\
&\leq \frac{\mu}{2}\|Z_k^{t,s} - Z_k^*\|^2
\end{align}

Combining these inequalities with our convergence rate:
\begin{align}
&\frac{\lambda}{2}\|Z_k^{t,s+1} - Z_k^*\|^2 \nonumber\\
&\leq  L_k(Z_k^{t,s+1}|\mathbf{A}_k^{R,t}, \mathbf{B}_k^{R,t}) - L_k(Z_k^*|\mathbf{A}_k^{R,t}, \mathbf{B}_k^{R,t})\\
&\leq (1 - 2\eta\lambda(1 - \frac{\mu\eta}{2}))(L_k(Z_k^{t,s}|\mathbf{A}_k^{R,t}, \mathbf{B}_k^{R,t}) \nonumber\\
& \quad- L_k(Z_k^*|\mathbf{A}_k^{R,t}, \mathbf{B}_k^{R,t}))\\
&\leq (1 - 2\eta\lambda(1 - \frac{\mu\eta}{2}))\frac{\mu}{2}\|Z_k^{t,s} - Z_k^*\|^2
\end{align}

Simplifying:
\begin{equation}
\|Z_k^{t,s+1} - Z_k^*\|^2 \leq (1 - 2\eta\lambda(1 - \frac{\mu\eta}{2}))\frac{\mu}{\lambda}\|Z_k^{t,s} - Z_k^*\|^2
\end{equation}

Denote $\rho = (1 - 2\eta\lambda(1 - \frac{\mu\eta}{2}))\frac{\mu}{\lambda}$, when we set the learning rate be $\frac{1}{\mu}$, we can make $\rho$ in the range of 0 to 1.

By applying this inequality recursively for $s$ steps:
\begin{equation}
\|Z_k^{t,s} - Z_k^*\|^2 \leq \rho^s \|Z_k^{t,0} - Z_k^*\|^2
\end{equation}

Taking the square root:
\begin{equation}
\|Z_k^{t,s} - Z_k^*(\mathbf{A}_k^{R,t}, \mathbf{B}_k^{R,t})\| \leq \rho^{s/2} \|Z_k^{t,0} - Z_k^*(\mathbf{A}_k^{R,t}, \mathbf{B}_k^{R,t})\|
\end{equation}

Since $\rho < 1$, we have $\lim_{s \to \infty} \rho^s = 0$. Therefore, with a sufficient number of local training steps $S$, we can ensure:
\begin{equation}
\|Z_k^{t+1} - Z_k^*(\mathbf{A}_k^{R,t}, \mathbf{B}_k^{R,t})\| \leq \epsilon
\end{equation}

where $Z_k^{t+1} = Z_k^{t,S}$ and $\epsilon$ can be made arbitrarily small by increasing $S$.

\textbf{Step 2: Global convergence across rounds}

Consider the difference between parameters at consecutive rounds:
\begin{align}
& \|Z_k^{t+1} - Z_k^{t}\| \nonumber\\
& \leq \|Z_k^{t+1} - Z_k^*(\mathbf{A}_k^{R,t}, \mathbf{B}_k^{R,t})\| \nonumber\\
&\quad+ \|Z_k^*(\mathbf{A}_k^{R,t}, \mathbf{B}_k^{R,t}) - Z_k^*(\mathbf{A}_k^{R,t-1}, \mathbf{B}_k^{R,t-1})\| \nonumber\\
&\quad+ \|Z_k^*(\mathbf{A}_k^{R,t-1}, \mathbf{B}_k^{R,t-1}) - Z_k^{t}\|
\end{align}

By Step 1 and Assumption 3:
\begin{equation}
\|Z_k^{t+1} - Z_k^{t}\| \leq 2\epsilon + \beta \|(\mathbf{A}_k^{R,t}, \mathbf{B}_k^{R,t}) - (\mathbf{A}_k^{R,t-1}, \mathbf{B}_k^{R,t-1})\|
\end{equation}

From the definition of RoW LoRA:
\begin{align}
&\|(\mathbf{A}_k^{R,t}, \mathbf{B}_k^{R,t}) - (\mathbf{A}_k^{R,t-1}, \mathbf{B}_k^{R,t-1})\| \nonumber\\
&= \frac{1}{K-1}\|\sum_{m\neq k}(Z_m^t - Z_m^{t-1})\| \\
&\leq \frac{1}{K-1}\sum_{m\neq k}\|Z_m^t - Z_m^{t-1}\|
\end{align}

Therefore:
\begin{equation}
\|Z_k^{t+1} - Z_k^{t}\| \leq 2\epsilon + \frac{\beta}{K-1}\sum_{m\neq k}\|Z_m^t - Z_m^{t-1}\|
\end{equation}

Let $\Delta^t = \max_k \|Z_k^{t+1} - Z_k^{t}\|$. Then:
\begin{equation}
\|Z_k^{t+1} - Z_k^{t}\| \leq 2\epsilon + \frac{\beta(K-1)}{K-1}\Delta^{t-1} = 2\epsilon + \beta\Delta^{t-1}
\end{equation}

Taking the maximum over all clients:
\begin{equation}
\Delta^t \leq 2\epsilon + \beta\Delta^{t-1}
\end{equation}

When $\beta < 1$, this is a contraction, and by iterating:
\begin{equation}
\Delta^t \leq 2\epsilon\sum_{i=0}^{t-1}\beta^i + \beta^t\Delta^0 \leq \frac{2\epsilon}{1-\beta} + \beta^t\Delta^0
\end{equation}

As $t \rightarrow \infty$, $\Delta^t \rightarrow \frac{2\epsilon}{1-\beta}$, which can be made arbitrarily small by reducing $\epsilon$ (increasing local training iterations).


When this condition is satisfied, FedALT converges to a stable solution within a neighborhood determined by the local training precision.

\end{proof}

\section{Limitations of FedALT in Traditional FL}
\label{appendix: limit}
In FedALT, a separate RoW LoRA is maintained alongside each client's Individual LoRA component. The RoW LoRA is not trained during each round but remains fixed, which slightly increases the forward/inference workload. However, we emphasize that FedALT is particularly well-suited for federated LoRA fine-tuning. Compared to the pre-trained model (e.g., LLaMA), the additional forward computation introduced by the RoW LoRA is minimal and can be largely neglected.

However, applying the Individual and RoW structure to traditional FL, where the full model is updated rather than a lightweight adapter, may not be practical. In such a setting, maintaining an additional RoW model for inference would effectively double the forward computation workload, making it computationally inefficient. Therefore, our method is specifically designed for federated fine-tuning scenarios, where the overhead remains manageable and the benefits of interference isolation are maximized.

\end{document}